\documentclass{sig-alternate-2013}

\usepackage[pass]{geometry}

\newfont{\mycrnotice}{ptmr8t at 7pt}
\newfont{\myconfname}{ptmri8t at 7pt}
\let\crnotice\mycrnotice%
\let\confname\myconfname%

\permission{Permission to make digital or hard copies of all or part of this work for personal or classroom use is granted without fee provided that copies are not made or distributed for profit or commercial advantage and that copies bear this notice and the full citation on the first page. Copyrights for components of this work owned by others than ACM must be honored. Abstracting with credit is permitted. To copy otherwise, or republish, to post on servers or to redistribute to lists, requires prior specific permission and/or a fee. Request permissions from Permissions@acm.org.}
\conferenceinfo{KDD'15,}{August 10-13, 2015, Sydney, NSW, Australia.} 
\copyrightetc{ACM \the\acmcopyr}
\crdata{ISBN 978-1-4503-3664-2/15/08\$15.00. \\
DOI: http://dx.doi.org/10.1145/2783258.2783340}
\makeatletter
\newif\if@restonecol
\makeatother

\usepackage[linesnumbered,boxed,commentsnumbered,ruled]{algorithm2e}

\newcommand{\hide}[1]{}

\usepackage{color}
\usepackage{soul}
\usepackage{morefloats}
\usepackage{pmat}
\usepackage{blkarray}
\usepackage{graphicx}
\usepackage{caption}
\usepackage{subcaption}
\usepackage{array}
\usepackage{multirow}
\usepackage{balance}
\newtheorem{problem}{Problem}
\newtheorem{lemma}{Lemma}
\newtheorem{theorem}{Theorem}

\newcommand{\ourAlgName}{\textsf{iBall}}
\newcommand{\matindex}[1]{\mbox{\scriptsize#1}}
\newcommand{\mat}[1]{{\bf #1}}   
\newcommand{\matsup}[2]{{\bf #1^{(#2)}}}
\newcommand{\hh}[1]{{\small\color{red}{\bf hh: #1}}}

\newcommand\myeq{\mathrel{\overset{\makebox[0pt]{\mbox{\normalfont\tiny\sffamily def}}}{=}}}
\renewcommand{\hh}[1]{}
\begin{document}
%

\title{The Child is Father of the Man: Foresee the Success \\ at the Early Stage}
%
%
%
%
%
\numberofauthors{2} 
\author{
\alignauthor
Liangyue Li\\
\affaddr{Arizona State University}\\
\email{liangyue@asu.edu}
\alignauthor
Hanghang Tong\\
\affaddr{Arizona State University}\\
\email{hanghang.tong@asu.edu}
}

\hide{
\numberofauthors{8} 
%
\author{
%
%
\alignauthor
Ben Trovato\titlenote{Dr.~Trovato insisted his name be first.}\\
       \affaddr{Institute for Clarity in Documentation}\\
       \affaddr{1932 Wallamaloo Lane}\\
       \affaddr{Wallamaloo, New Zealand}\\
       \email{trovato@corporation.com}
\alignauthor
G.K.M. Tobin\titlenote{The secretary disavows
any knowledge of this author's actions.}\\
       \affaddr{Institute for Clarity in Documentation}\\
       \affaddr{P.O. Box 1212}\\
       \affaddr{Dublin, Ohio 43017-6221}\\
       \email{webmaster@marysville-ohio.com}
\alignauthor Lars Th{\o}rv{\"a}ld\titlenote{This author is the
one who did all the really hard work.}\\
       \affaddr{The Th{\o}rv{\"a}ld Group}\\
       \affaddr{1 Th{\o}rv{\"a}ld Circle}\\
       \affaddr{Hekla, Iceland}\\
       \email{larst@affiliation.org}
\and  
\alignauthor Lawrence P. Leipuner\\
       \affaddr{Brookhaven Laboratories}\\
       \affaddr{Brookhaven National Lab}\\
       \affaddr{P.O. Box 5000}\\
       \email{lleipuner@researchlabs.org}
\alignauthor Sean Fogarty\\
       \affaddr{NASA Ames Research Center}\\
       \affaddr{Moffett Field}\\
       \affaddr{California 94035}\\
       \email{fogartys@amesres.org}
\alignauthor Charles Palmer\\
       \affaddr{Palmer Research Laboratories}\\
       \affaddr{8600 Datapoint Drive}\\
       \affaddr{San Antonio, Texas 78229}\\
       \email{cpalmer@prl.com}
}
\additionalauthors{Additional authors: John Smith (The Th{\o}rv{\"a}ld Group,
email: {\texttt{jsmith@affiliation.org}}) and Julius P.~Kumquat
(The Kumquat Consortium, email: {\texttt{jpkumquat@consortium.net}}).}
\date{30 July 1999}
}

\maketitle

\begin{abstract}

Understanding the dynamic mechanisms that drive the high-impact scientific work (e.g., research papers, patents) is a long-debated research topic and has many important implications, ranging from personal career development and recruitment search, to the jurisdiction of research resources. Recent advances in characterizing and modeling scientific success have made it possible to forecast the long-term impact of scientific work, where data mining techniques, supervised learning in particular, play an essential role. Despite much progress, several key algorithmic challenges in relation to predicting long-term scientific impact have largely remained open. In this paper, we propose a joint predictive model to forecast the long-term scientific impact at the early stage, which simultaneously addresses a number of these open challenges, including the scholarly feature design, the non-linearity, the domain-heterogeneity and dynamics. In particular, we formulate it as a regularized optimization problem and propose effective and scalable algorithms to solve it. We perform extensive empirical evaluations on large, real scholarly data sets to validate the effectiveness and the efficiency of our method.

\hide{

Understanding the underlying pattern of items' impact can help with people's attention economy.
In this paper, we focus on predicting individual paper's future citation count by looking only at its first three years' citation history.
Prior studies assume that all the papers' citations are governed by the same underlying process.
Instead of learning a single regression model for the whole paper corpus, we partition the papers based on their citation behaviors and learn a regression model for each domain.
To leverage the similarity of closely related domains, we propose to jointly learn these regression models. The framework can be generalized to linear and non-linear regression and 
enjoys closed-form solution. To tackle the computation challenges of the non-linear learning algorithm, we design a fast approximation algorithm using low-rank matrix approximation techniques. 
We also give the correctness proof and error bound analysis of our algorithms.
The experimental results on real world citation dataset clearly show the effectiveness and efficiency of our algorithms. 
}
\end{abstract}



\category{H.2.8}{Database Management}{Database applications}[Data mining]
\keywords{long term impact prediction; joint predictive model}

\section{Introduction}
\label{sec:intro}
Understanding the dynamic mechanisms that drive the high-impact scientific work (e.g., research papers, patents) is a long-debated research topic and has many important implications, ranging from personal career development and recruitment search, to the jurisdiction of research resources. Scholars, especially junior scholar, who could master the key to producing high-impact work would attract more attentions as well as research resources; and thus put themselves in a better position in their career developments. High-impact work remains as one of the most important criteria for various organization (e.g. companies, universities and governments) to identify the best talents, especially at their early stages. It is highly desirable for researchers to judiciously search the right literature that can best benefit their research. 

\hide{
The impact/popularity of an item reflects how much influence it has on the general public.
For example, Ellen DeGeneres' selfie with other celebrities at Oscars 2014 set the record of most retweeted tweet of all time. 
In December 2014, Psy's Gangnam Style music video has been viewed so many times that Youtube had to upgrade to a 64-bit integer.
The question ``Lifestyle: What can I learn/know right now in 10 minutes that will be useful for the rest of my life?" has more than 10 million views on Quora. In these examples, the number of retweets and views reflect the impact of the tweets, videos and questions. The items with high impact often have more information, importance, and insights and can spur more discussion, comments and inspiration. Understanding the pattern of impact can not only prevent us from drowning in information explosion, but also is of importance for policy maker, marketing practitioner, etc.
}

Recent advances in characterizing and modeling scientific success have made it possible to forecast the long-term impact of scientific work. Wuchty et al.~\cite{Wuchty18052007} observe that papers with multiple authors receive more citations than solo-authored ones. 
Uzzi et al.~\cite{Uzzi25102013} find that the highest-impact science work is primarily grounded in atypical combinations of prior ideas while embedding them in conventional knowledge frames. Recently, Wang et al. ~\cite{wang2013quantifying} develop a mechanistic model for the citation dynamics of individual papers. In data mining community, efforts have also been made to predict the long-term success. Carlos et al.~\cite{castillo2007estimating} estimate the number of citations of a paper based on the information of past articles written by the same author(s). Yan et al.~\cite{DBLP:conf/cikm/YanTLSL11} design effective content (e.g., topic diversity) and contextual (e.g., author's {\it h}-index)\hh{liangyue: what kind of features, e.g., content-based, contextual features? mention that} features for the prediction of future citation counts.  Despite much progress, the following four key algorithmic challenges in relation to predicting long-term scientific impact have largely remained open.
\vspace{-5pt}
\begin{itemize}
\setlength\itemsep{1pt}
\item [C1] {\it Scholarly feature design:} many factors could affect scientific work's long-term impact, e.g., research topics, author reputations, venue ranks, citation networks' topological features, etc. Among them, which bears the most predictive power? 
\item [C2] {\it Non-linearity:} the effect of the above scholarly features on the long-term scientific impact might be way beyond a linear relationship.
\item [C3] {\it Domain heterogeneity:} the impact of scientific work in different fields or domains might behave differently; yet some closely related fields could still share certain commonalities. Thus, a one-size-fits-all or one-size-fits-one solution  might be sub-optimal.
\item [C4] {\it Dynamics:} with the rapid development of science and engineering, a significant number of new research papers are published each year, even on a daily basis with the advent of arXiv\footnote{arxiv.org}. The predictive model needs to handle such stream-like data efficiently, to reflect the recency of the scientific work.
\end{itemize}

In this paper, we propose a joint predictive model--\underline{I}mpact Crystal \underline{Ball} (\ourAlgName~in short)  -- to forecast the long term scientific impact at an early stage by collectively addressing the above four challenges. First (for C1), we found that the citation history of a scholarly entity (e.g., paper, researcher, venue) in the first three years (e.g., since its publication date) is a strong indicator of its long-term impact (e.g., the accumulated citation count in ten years); and adding additional contextual or content features brings little marginal benefits in terms of prediction performance. This not only largely simplifies the feature design, but also enables us to forecast the long-term scientific impact at its early stage. 
Second (for C2), our joint predictive model is flexible, being able to characterize both the linear and non-linear relationship between the features and the impact score. Third (for C3), we propose to jointly learn a predictive model to differentiate distinctive domains, while taking into consideration the commonalities among these similar domains (see an illustration in Figure~\ref{fig:domain_relation}). 
 Fourth (for C4), we further propose a fast on-line update algorithm to adapt our joint predictive model efficiently over time to accommodate newly arrived training examples (e.g., newly published papers).

\hide{
The challenges in predicting a paper's future citation count come in two folds: first, the distribution of citation counts exhibits heavy tail. That is, majority of the papers
have no or very few citations. For example, in AMiner citation network dataset~\cite{DBLP:conf/kdd/TangZYLZS08}, 87.45\% of the papers have zero citation 10 years after publication. 
Second, the dynamic process that governs each individual paper is too noisy to be amenable to quantification as noted in~\cite{DBLP:conf/aaai/ShenWSB14,wang2013quantifying}. Besides, the papers' yearly citations exhibits drastically complex patterns.   For instance, some papers have no or few citations throughout their life cycle, some papers receive intense 
attentions in certain years after they publish. Such two challenges make a single regression model incompetent for predicting a paper's citation count as in the pilot study~\cite{DBLP:conf/cikm/YanTLSL11}.

In this paper, we propose to stratify the
papers according to their citation behavior observed within the first three years of publication and then learn a regression model in each strata. To further improve the prediction within each strata, we design a joint learning framework to learn a regression model for each strata together. The reasoning is that strata with similar citation behavior would share similar learning parameters and such knowledge
should be leveraged across strata. The benefit is especially pronounced for strata with very few training samples. For example, in Figure~\ref{fig:domain_relation}, we use nodes to represent different domains and each domain has some papers represented by their citation history. The strength of links between domains indicate the similarity between the two domains. In this figure,  Domain 1, 2 and 3 have triangular relationship, intuitively the processes that govern the citation behavior in each of these three domains might be similar; while Domain 4 is only connected Domain 2, its citation pattern is only similar to that of Domain 2. Our joint learning model will enforce such domain relations.

We show in the paper that the joint learning model can be generalized to both linear and non-linear models. The experiment results shows that the non-linear models exhibit exceptionally good prediction accuracy, however, its power is limited by the large space consumption and low efficiency. To speed up the joint non-linear method, we carefully design an approximation algorithm taking advantage of low-rank approximation in matrix computation. The experiment shows ...
}

\begin{figure}[!tb]
\centering
\includegraphics[width = 0.45\textwidth, height=50mm ]{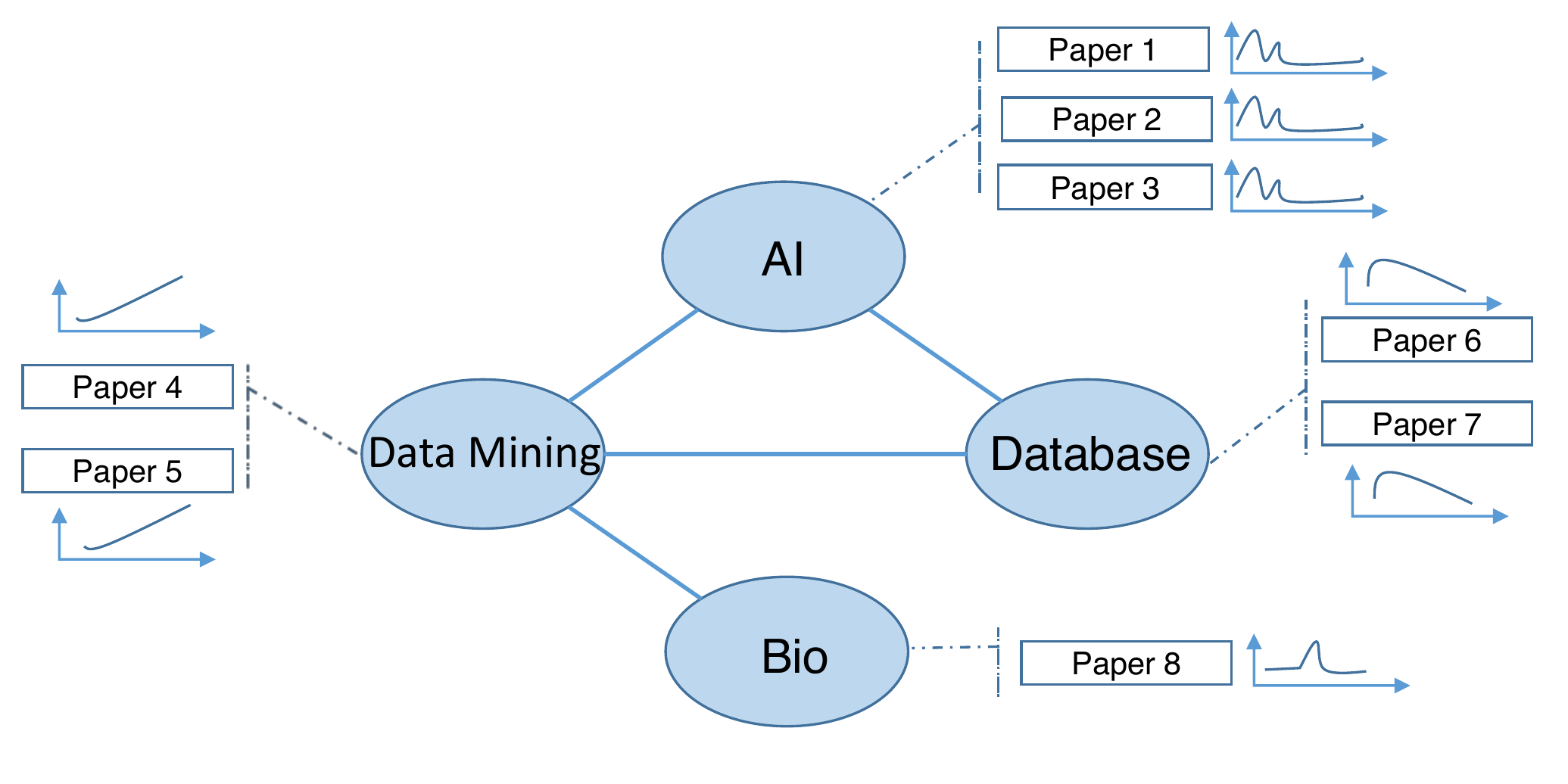}
\caption{An illustrative example of the proposed joint predictive model. Papers from the same domain (e.g., AI, Databases, Data Mining and Bio) share similar patterns in terms of attracting citations over time. Certain domains (e.g., AI and Data Mining) are more related with each other than other domains (e.g., AI and Bio). We want to jointly learn four predictive models (one for each domain), with the goal of encouraging the predictive models from more related domains (e.g., AI and Data Mining) to be `similar' with each other.}\label{fig:domain_relation}
\end{figure}

Our main contributions can be summarized as follows:
\begin{itemize}
\item {\bf Algorithms:} we propose a joint predictive model --\ourAlgName -- for the long-term scientific impact prediction problem, together with its efficient solvers.

\item {\bf Proofs and analysis:} we analyze the correctness, the approximation quality and the complexity of our proposed algorithms.

\item {\bf Empirical evaluations:} we conduct extensive experiments to demonstrate the effectiveness and efficiency of our proposed algorithms. 
\end{itemize}

The rest of the paper is organized as follows. Section~\ref{sec:prob} gives the problem definition. Section~\ref{sec:empirical_obs} provides empirical observation of the AMiner citation network dataset. 
Section~\ref{sec:algs} proposes our joint model and the fast algorithm.
Section~\ref{sec:exp} shows the experimental results. Section~\ref{sec:rel} reviews related work and the paper concludes in Section~\ref{sec:con}.

\section{Problem Statement}
\label{sec:prob}
\begin{table}[!t]
\caption{Symbols}
\centering
\begin{tabular}{| c | p{7cm}|}
\hline
\bf{Symbols}                & \bf{Definition}\\
\hline \hline
$n_d$  & number of domains\\
$n_i$  & number of training samples in the $i$-th domain\\
$m_i$  & number of new training samples in the $i$-th domain\\
$d$    & feature dimensionality\\
$\matsup{X_t}{i}$ & feature matrix of training samples from the $i$-th domain at time $t$\\
$\matsup{x_{t+1}}{i}$ & feature matrix of new training samples from the $i$-th domain at time $t+1$\\
$\matsup{Y_t}{i}$ & impact vector of training samples from the $i$-th domain at time $t$\\
$\matsup{y_{t+1}}{i}$ & impact vector of new training samples from the $i$-th domain at time $t+1$\\
$\mat{A}$ & adjacency matrix of domain relation graph\\
$\matsup{w}{i}$ & model parameter for the $i$-th domain\\
$\matsup{K}{i}$ & kernel matrix of training samples in the $i$-th domain\\
$\matsup{K}{ij}$ & cross domain kernel matrix of training samples in the $i$-th and $j$-th domains\\
\hline
\end{tabular}
\label{tab:symbol}
\end{table}

In this section, we first present the notations used throughout the paper and then formally define the long-term scientific impact prediction for scholarly entities (e.g., research papers, researchers, conferences).

Table~\ref{tab:symbol} lists the main symbols used throughout the paper. We use bold capital letters (e.g., $\mat{A}$) for matrices, bold lowercase letters (e.g., $\mat{w}$) for vectors, 
and lowercase letters (e.g., $\lambda$) for scalars. 
For matrix indexing, we use a convention similar to Matlab as follows, e.g., 
we use $\mat{A}(i,j)$ to denote the entry at the $i$-th row and $j$-th column of a matrix $\mat{A}$,
 $\mat{A}(i, :)$ to denote the $i$-th row of $\mat{A}$ and $\mat{A}(:, j)$ to denote the $j$-th column of $\mat{A}$. Besides, we use prime for matrix transpose, e.g., $\mat{A'}$ is the transpose of $\mat{A}$. 
 
 To differentiate samples from different domains at different time steps, we use superscript to index the domain and subscript to indicate timestamp.
 For instance, $\matsup{X_t}{i}$ denotes the feature matrix of all the scholarly entities in the $i$-th domain at time $t$ and $\matsup{x_{t+1}}{i}$ denotes the feature matrix of new scholarly entities in the $i$-th domain at time $t+1$. Hence, $\matsup{X_{t+1}}{i} = [\matsup{X_t}{i};\matsup{x_{t+1}}{i}]$. 
 Similarly, $\matsup{Y_t}{i}$ denotes the impact vector of scholarly entities in the $i$-th domain at time $t$ and $\matsup{y_{t+1}}{i}$ denotes the impact vector of new scholarly entities in the $i$-th domain at time $t+1$. Hence, $\matsup{Y_{t+1}}{i} = [\matsup{Y_t}{i};\matsup{y_{t+1}}{i}]$. We will omit the superscript and/or subscript when the meaning of the matrix is clear from the context. 
 
 With the above notations, we are ready to define the long-term impact prediction problem in both static and dynamic settings as follows:

\begin{problem}{Static Long-term Scientific Impact Prediction}

\begin{description}
\item[Given:] feature matrix $\mat{X}$ and impact $\mat Y$ of scholarly entities

\item[Predict:] the long-term impact of new scholarly entities
\end{description}
\end{problem}

We further define the dynamic impact prediction problem as:
\begin{problem}{Dynamic Long-term Scientific Impact Prediction}

\begin{description}
\item[Given:] feature matrix $\mat{X_t}$ and new training feature matrix $\mat{x_{t+1}}$ of scholarly entities, the impact vector $\mat{Y_t}$, and the impact vector of new training samples $\mat{y_{t+1}}$

\item[Predict:] the long-term impact of new scholarly entities
\end{description}

\end{problem}

\section{Empirical Observations}
\label{sec:empirical_obs}
\begin{figure*}[!htb]
\minipage[t]{0.32\textwidth}
\centering
  \includegraphics[width=\linewidth]{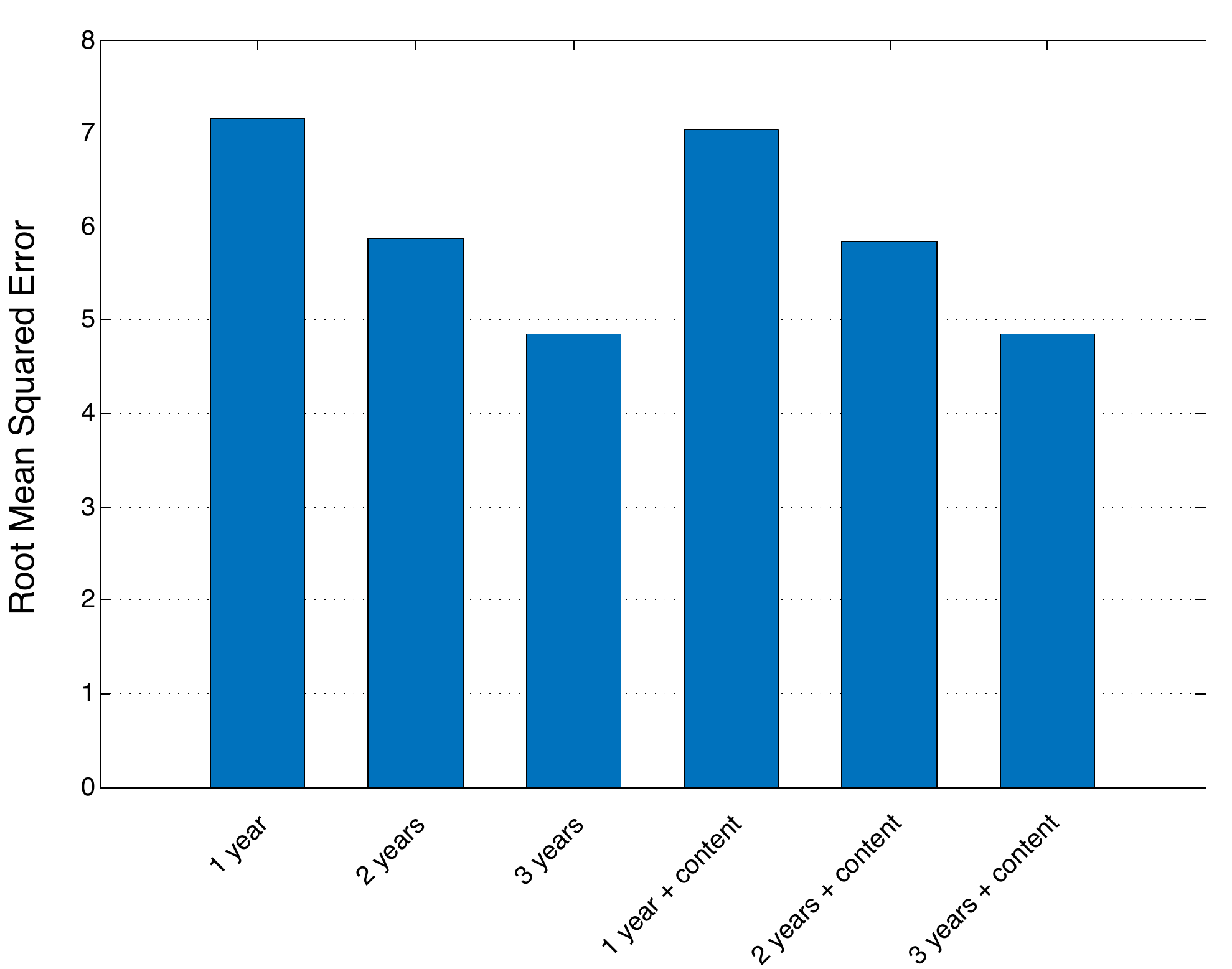}
\caption{Prediction error comparison with different features.}\label{fig:feature_design}
\endminipage\hfill
\minipage[t]{0.32\textwidth}
\centering
  \includegraphics[width=\linewidth]{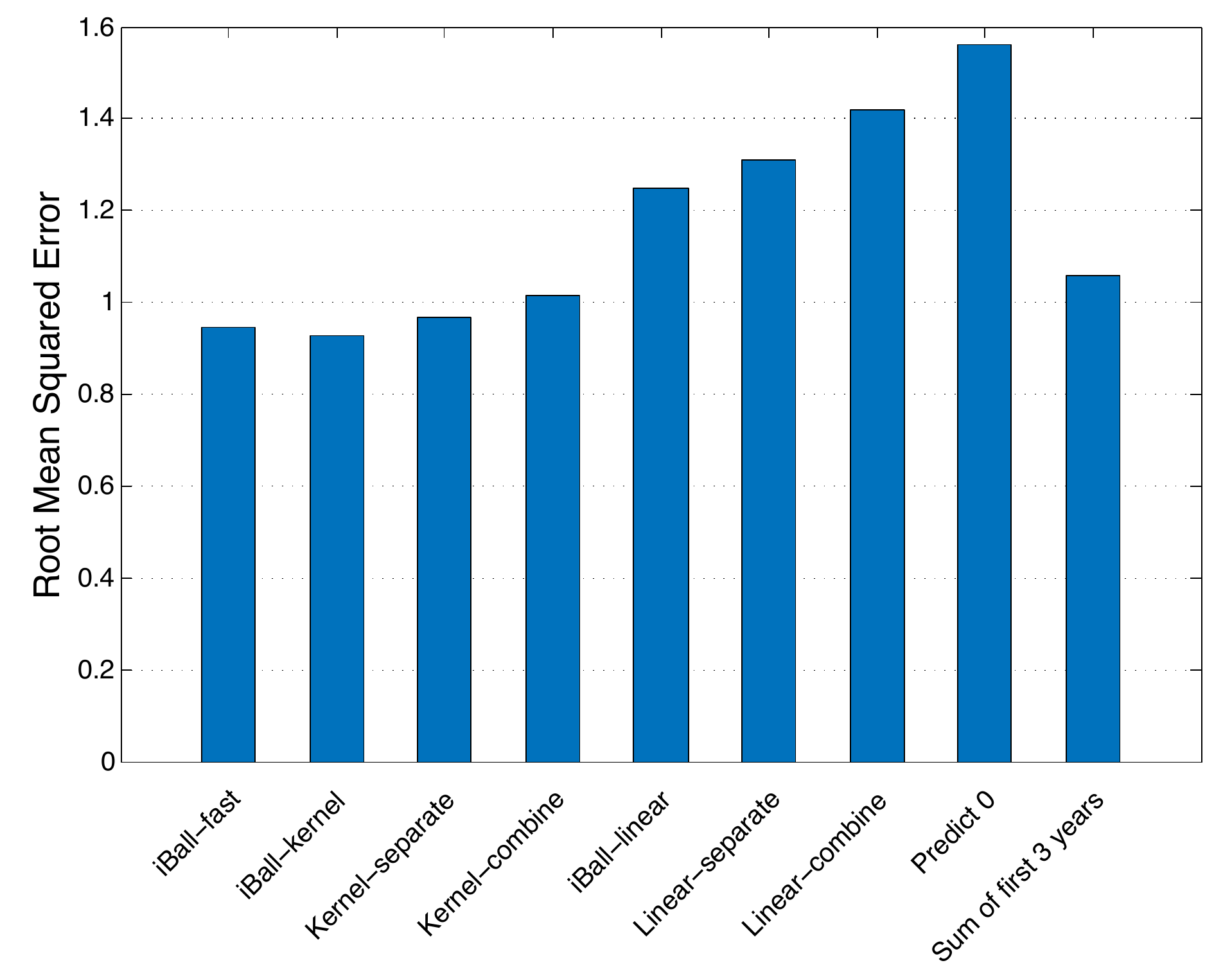}
  \caption{RMSE comparisons using different methods. The citation count is normalized in this figure. See Section~\ref{sec:exp} for normalization details.}\label{fig:non_linearity}
\endminipage\hfill
\minipage[t]{0.32\textwidth}
\centering
  \includegraphics[width=\linewidth]{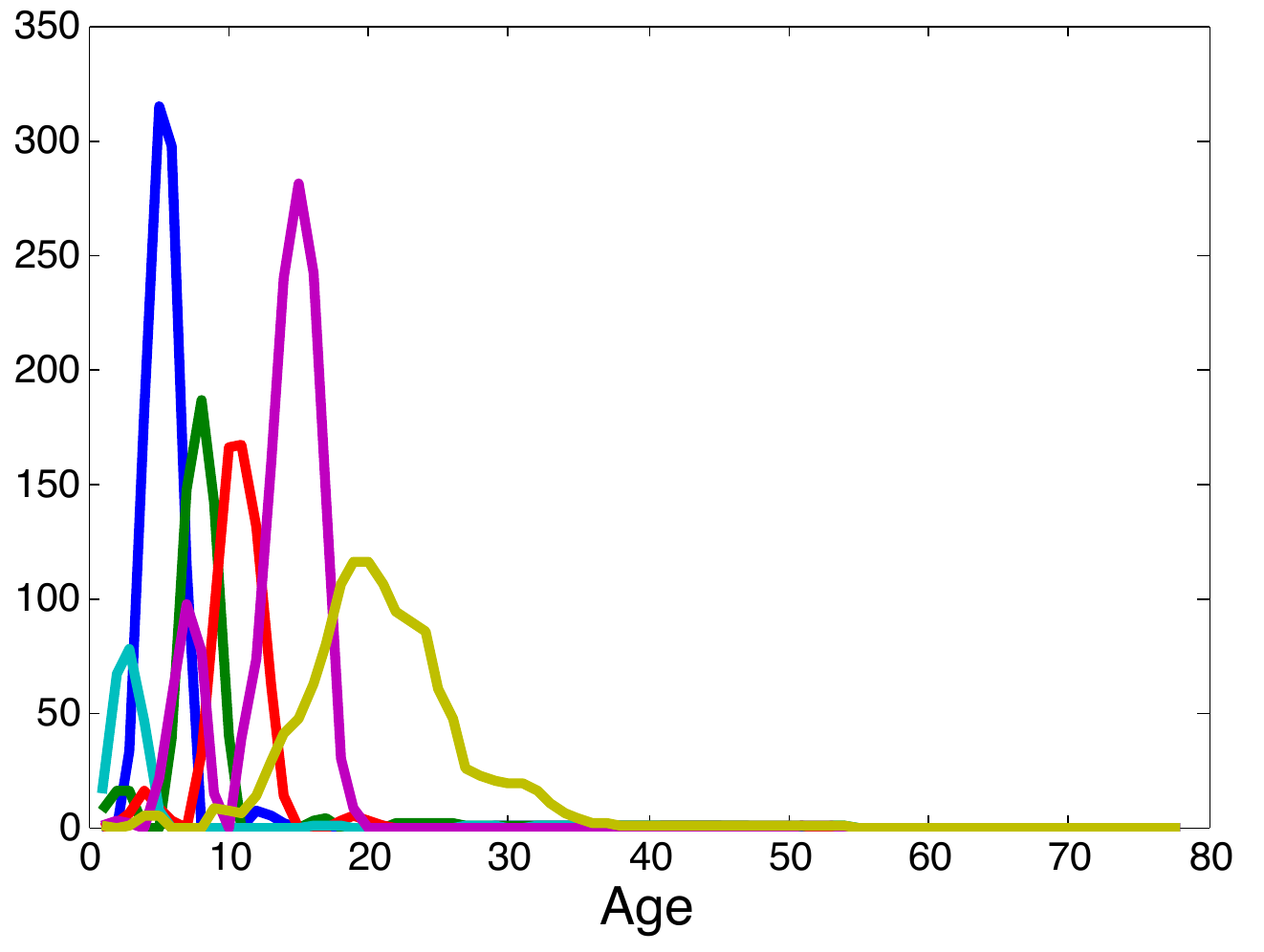}
  \caption{Visualization of papers' citatoin behavior. Different colors encodes different citation behaviors.}\label{fig:citation_pattern}
\endminipage
\end{figure*}

In this section, we perform an empirical analysis to highlight some of the key challenges (summarized in introduction section), on {\em AMiner} citation network~\cite{DBLP:conf/kdd/TangZYLZS08}. This is a rich real dataset for bibliography network analysis and mining. The dataset contains 2,243,976 papers, 1,274,360 authors, and 8,882 computer science venues. For each paper, the dataset provides its titles, authors, references, publication venue and publication year. The papers date from year 1936 to 2013. In total, the dataset has 1,912,780 citation relationships extracted from ACM library. 

\subsection{Power-law distribution}

The distribution of the citation counts of all the papers and the distribution of the number of citations received within 10 years after publication are presented in Figures~\ref{subfig:paper_citation} and \ref{subfig:paper_10yr_citation}. We also show the distribution of citation counts of all the authors and all the venues respectively in Figures~\ref{subfig:author_citation} and \ref{subfig:venue_citation}. It is clear that all these citations are of a power law distribution. Nearly 87.45\% papers have zero citations within 10 years. 


\begin{figure}[!htb]
\centering
\begin{subfigure}[t]{0.2\textwidth}
\includegraphics[width=\textwidth]{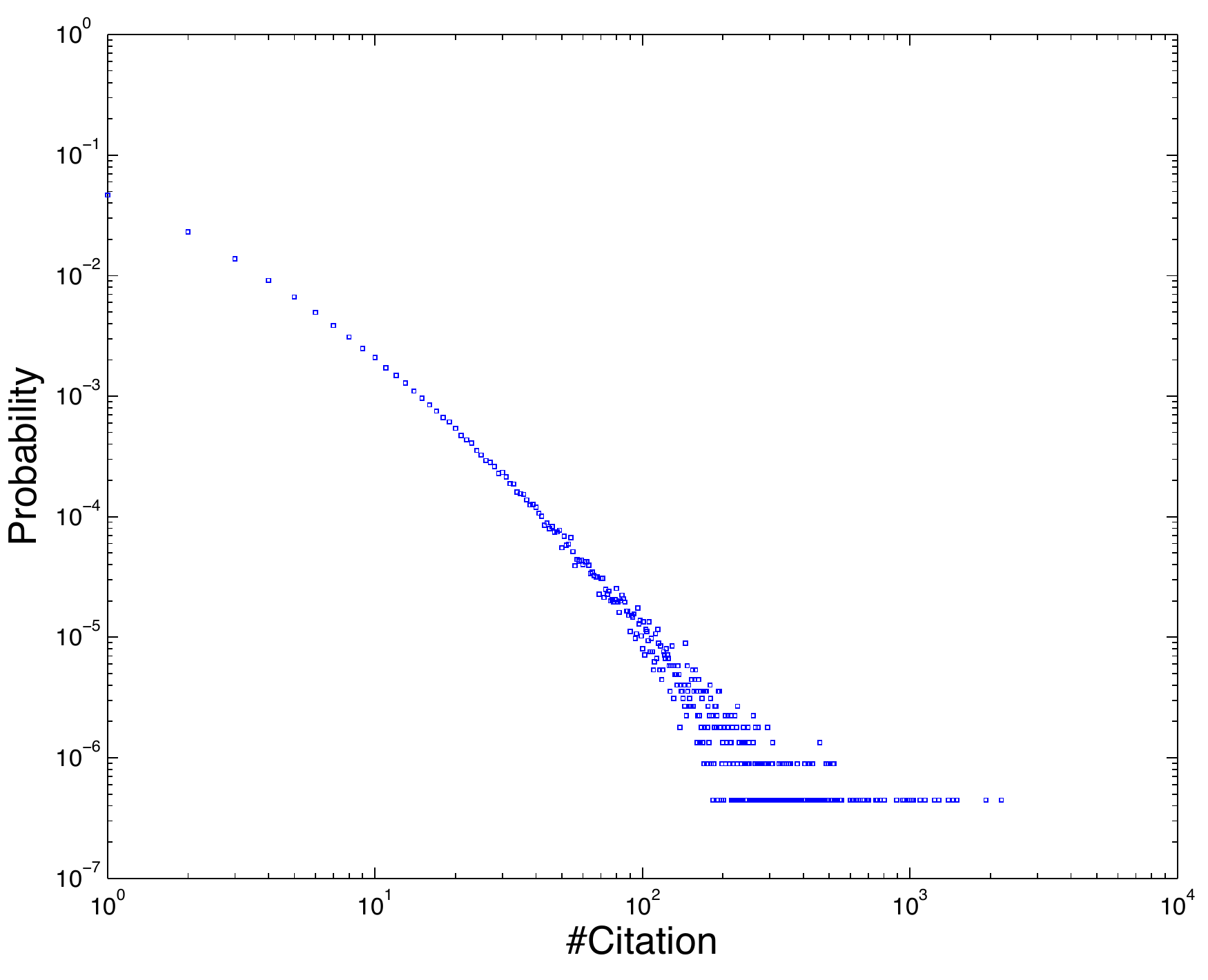}
\caption{Distribution of the number of all citations of papers.}
\label{subfig:paper_citation}
\end{subfigure}
~
\begin{subfigure}[t]{0.2\textwidth}
\includegraphics[width=\textwidth]{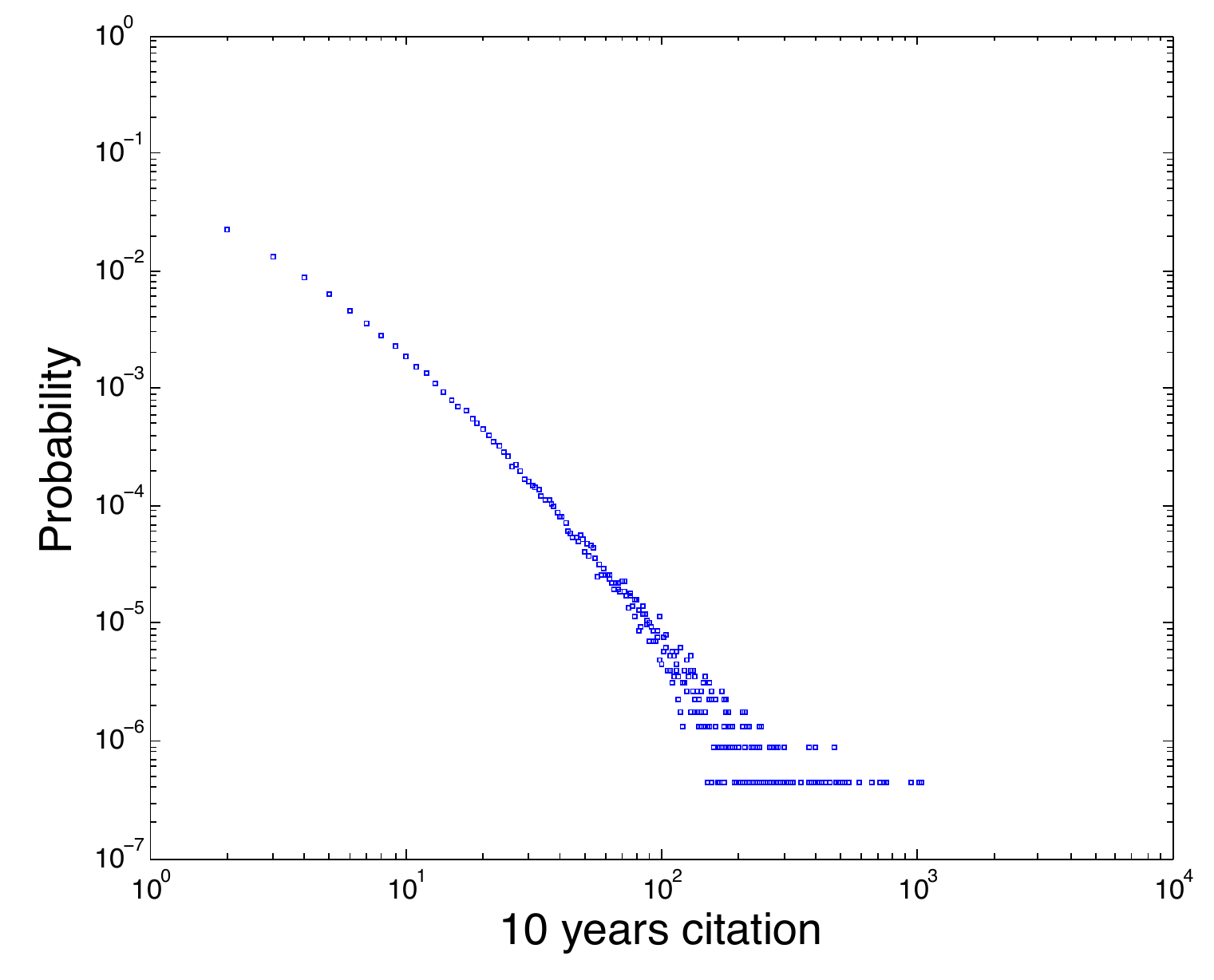}
\caption{Distribution of the number of citations of papers received within 10 years after publication.}
\label{subfig:paper_10yr_citation}
\end{subfigure}

\begin{subfigure}[t]{0.2\textwidth}
\includegraphics[width=\textwidth]{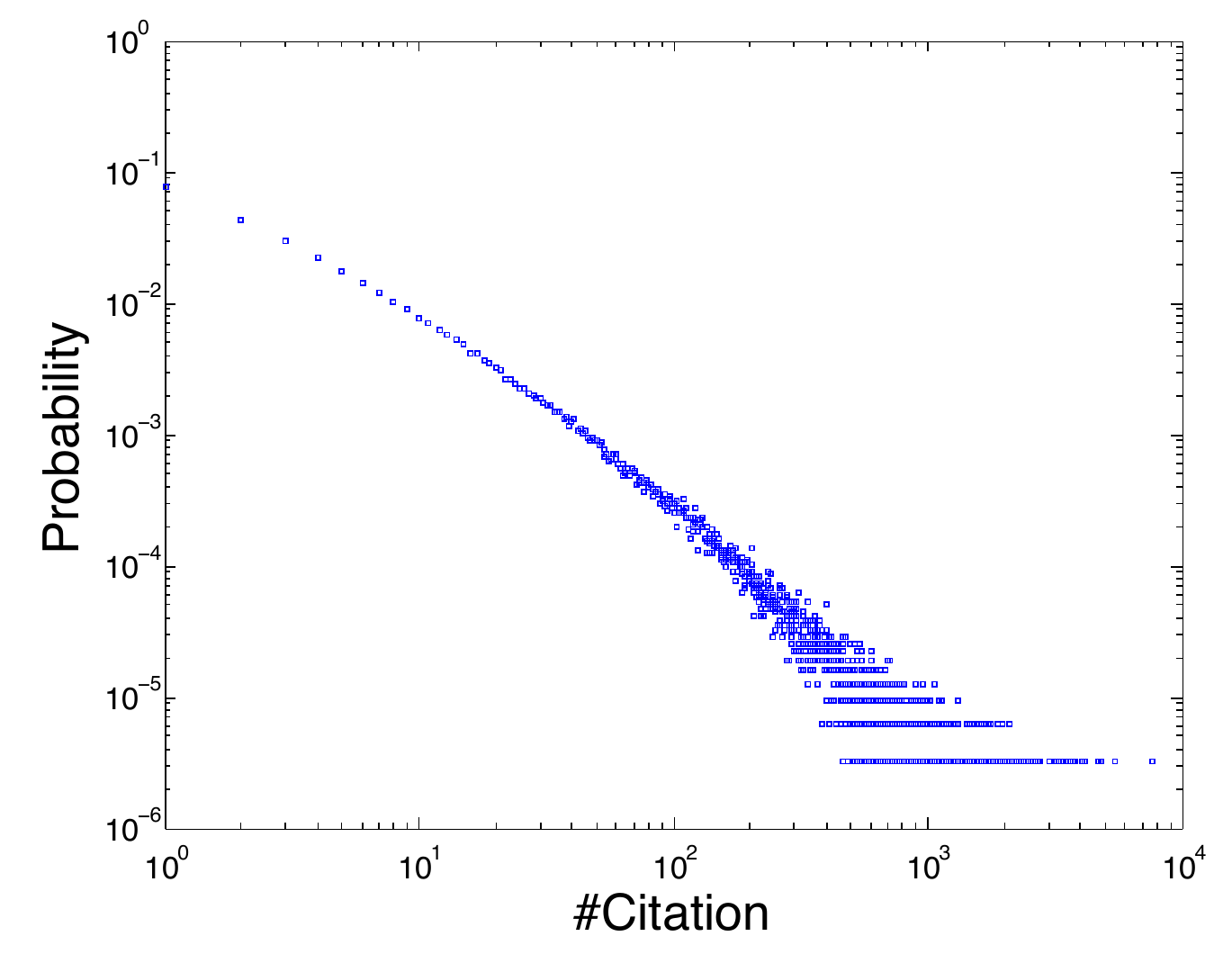}
\caption{Distribution of the number of all citations of authors.}
\label{subfig:author_citation}
\end{subfigure}
~
\begin{subfigure}[t]{0.2\textwidth}
\includegraphics[width=\textwidth]{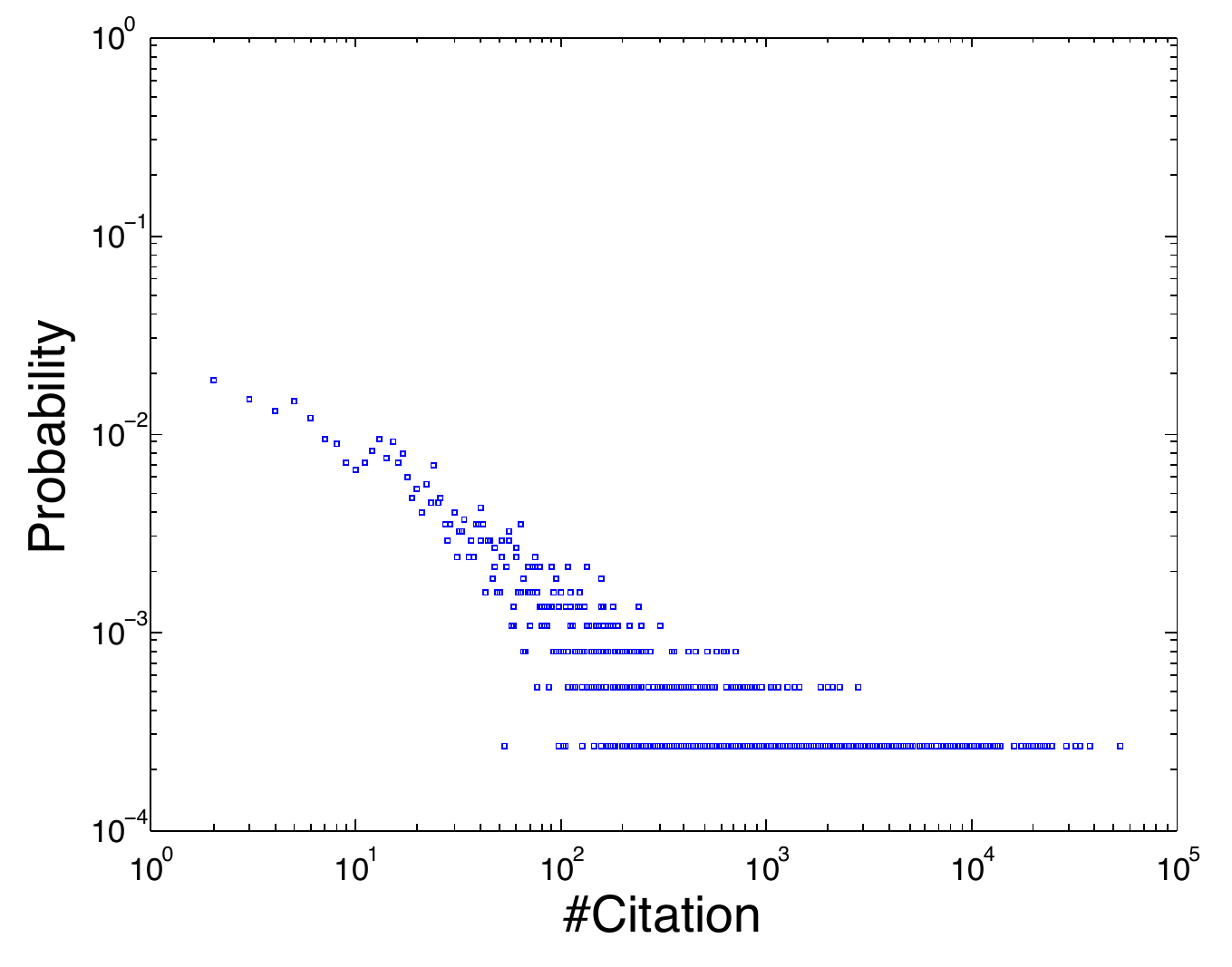}
\caption{Distribution of the number of all citations of venues}
\label{subfig:venue_citation}
\end{subfigure}
\caption{Citation distributions of AMiner citation dataset.}
\label{fig:citation_dist}
\end{figure}

\vspace{-5pt}
\subsection{Feature design}
Prior work~\cite{castillo2007estimating,DBLP:conf/cikm/YanTLSL11} has proposed some effective features for citation count prediction, e.g., topic features (topic rank, diversity), author features ({\it h}-index, productivity), venue features (venue rank, venue centrality). Other work~\cite{wang2013quantifying} make predictions only on the basis of the early years' citation data and find that the future impact of majority papers fall within the predicted citation range. We conduct experiment to compare performance of different features. Figure~\ref{fig:feature_design} shows the root mean squared error using different features with a regression model for the prediction of 10 years' citation count. For example, `3 years' means using the first 3 years' citation as feature, and `3 years + content' means using the first 3 years' citation along with content features (e.g., topic, author features). The result shows that adding content features (the right three bars in the figure) brings little improvement for citation prediction.



\subsection{Non-linearity}
To see if the feature has linear relationship with the citation, we compare the performance of different methods using only the first 3 years' citation history. In Figure~\ref{fig:non_linearity}, the non-linear models (\ourAlgName-fast, \ourAlgName-kernel, Kernel-combine) all outperform the linear models (\ourAlgName-linear, Linear-separate, Linear-combine). See Section~\ref{sec:algs} and \ref{sec:exp} for details of these models. It is clear that complex relationship between the features and the impact cannot be well characterized by a simple linear model - the prediction performance for all the linear models is even worse than the baseline method (using the summation of the first 3 years' citation counts).


\subsection{Domain heterogeneity}
To get a sense of the dynamic patterns of the citation count, we construct a paper-age citation matrix $\mat{M}$, where $\mat{M_{ij}}$ indicates the number of citations the $i$-th paper receives in the $j$-th year after it gets published. The matrix $\mat{M}$ is then factorized as $\mat{M} \approx \mat{WH}$ using Non-negative Matrix Factorization (NMF)~\cite{lee1999learning}. We visualize the first six rows of $\mat{H}$ in Figure~\ref{fig:citation_pattern}, which can give us different clustering citation dynamic patterns. As can be seen from the figure, the cyan line has a very small peak in the first 3 years and then fades out very quickly; the blue line picks up very fast in the early years and then fades out; the yellow line indicates a delayed pattern where the scientific work only receives some amount of attentions decades after it gets published. This highlights that impact of scientific work from different domains behaves differently.



\section{Proposed algorithms}
\label{sec:algs}
%
%
%
%
%
%

In this section, we present our joint predictive model to forecast the long-term scientific impact at an early stage. We first formulate it as a regularized optimization problem; then propose effective, scalable and adaptive algorithms; followed up by theoretical analysis in terms of the optimality, the approximation quality as well as the computational complexity. 

\subsection{iBall -- Formulations}
Our predictive model applies to different types of scholarly entities (e.g., papers, researchers and venues). For the sake of clarity, we will use paper citation prediction as an example. As mentioned earlier, research papers are in general from different domains. We want to jointly learn a predictive model for each of the domains, with the design objective to leverage the commonalities between related domains. Here, the commonalities among different domains is described by a non-negative $\mat{A}$, i.e., if the $i$-th and $j$-th domains are closely related, its corresponding $\mat{A}_{ij}$ entry will have a higher numerical value. Denote feature matrix for papers in the $i$-th domain by $\matsup{X}{i}$, citation count of papers in the $i$-th domain by $\mat{Y^{(i)}}$ and the model parameter for the $i$-th domain by $\matsup{w}{i}$, we have the following joint predictive model\hh{liangyue: replace all the prediction model by predictive model}:

\begin{equation}
\begin{array}{rl}
\min\limits_{\matsup{w}{i}, i=1,\ldots, n_d} & \sum\limits_{i=1}^{n_d} \mathcal{L}[f(\matsup{X}{i},\matsup{w}{i}), \matsup{Y}{i}]  \\
 & + \theta\sum\limits_{i=1}^{n_d} \sum\limits_{j=1}^{n_d} \mat{A_{ij}}g(\matsup{w}{i},\matsup{w}{j})+ \lambda \sum\limits_{i=1}^{n_d} \Omega(\matsup{w}{i})\\
\end{array}
\label{eq:joint_model}
\end{equation}
\noindent where $f(\matsup{X}{i},\matsup{w}{i})$ is the prediction function for the $i$-th domain, $\mathcal{L}(.)$ is a loss function, $g(\matsup{w}{i},\matsup{w}{j})$ characterizes the relationship between the model parameters of the $i$-th and $j$-th domains, $\Omega(\matsup{w}{i})$ is the regularization term for model parameters and $\theta$, $\lambda$ are regularization parameters to balance the relative importance of each aspect.

As can be seen, this formulation is quite flexible and general. Depending on the loss function we use, our predictive model can be formulated as regression or classification task. Depending on the prediction function we use, we can have either linear or non-linear models. The core of our joint model is the second term that relates parameters of different models. If $\mat{A}_{ij}$ is large, meaning the $i$-th and $j$-th domains are closely related to each other, we want the function value $g(.)$ that characterizes the relationship between the parameters to be small. 


{\it \ourAlgName -- linear formulation:} if the feature and the output can be characterized by a linear relationship, we can use a linear function as the prediction function and the Euclidean distance for the distance between model parameters. The linear model can be formulated as follows:
\begin{equation}
\begin{array}{rl}
\min\limits_{\matsup{w}{i}, i=1,\ldots, n_d} & \sum\limits_{i=1}^{n_d} \| \matsup{X}{i} \matsup{w}{i} - \matsup{Y}{i} \|_2^2  \\
 & + \theta\sum\limits_{i=1}^{n_d} \sum\limits_{j=1}^{n_d} \mat{A_{ij}}\|\matsup{w}{i} - \matsup{w}{j} \|_2^2+ \lambda \sum\limits_{i=1}^{n_d} \|\matsup{w}{i}\|_2^2\\
\end{array}
\label{eq:joint_linear}
\end{equation}
\noindent where $\theta$ is a balance parameter to control the importance of domain relations, and $\lambda$ is a regularization parameter. From the above objective function we can see that, if the $i$-th domain and $j$-th domain are closely related, i.e., $\mat{A}_{ij}$ is a large positive number, it encourages a smaller Euclidean distance between $\matsup{w}{i}$ and $\matsup{w}{j}$. The intuition is that for a given feature, it would have a similar effect in predicting the papers from two similar/closely related domains.

{\it \ourAlgName -- non-linear formulation:} As indicated in our empirical studies (Figure~\ref{fig:non_linearity}), the relationship between the features and the output (citation counts in ten years) is far beyond linear. Thus, we further develop the kernelized counterpart of the above linear model. Let us denote the kernel matrix of papers in the $i$-th domain by $\matsup{K}{i}$, which can be computed as $\matsup{K}{i}(a,b) = k(\matsup{X}{i}(a,:), \matsup{X}{i}(b,:))$, where $k(\cdot, \cdot)$ is a kernel function that implicitly computes the inner product in a high-dimensional reproducing kernel Hilbert space (RKHS)~\cite{aronszajn1950theory} . Similarly, we define the cross-domain kernel matrix by $\matsup{K}{ij}$, which can be computed as $\matsup{K}{ij}(a,b) = k(\matsup{X}{i}(a,:) , \matsup{X}{j}(b,:))$, reflecting the similarity between papers in the $i$-th domain and $j$-th domain. Different from the linear model where the model parameters in different domains share the same dimensionality (i.e., the dimensionality of the raw feature), in the non-linear case, the dimensionality of the model parameters are the same as the number of training samples in each domain, which is very likely to be different across different domains. Thus, we cannot use the same distance function for $g(.)$. To address this issue, the key is to realize that the predicted value of a test sample using kernel methods is a linear combination of the similarities between the test sample and all the training samples. Therefore, instead of restricting the model parameters to be similar, we impose the constraint that the predicted value of a test sample using the training samples in its own domain and using training samples in a closely related domain to be similar. The resulting non-linear model can be formulated as follows:
\vspace{-5pt}
\begin{equation}
\begin{array}{rl}
\min\limits_{\matsup{w}{i}, i=1,\ldots, n_d} & \sum\limits_{i=1}^{n_d} \| \matsup{K}{i} \matsup{w}{i} - \matsup{Y}{i} \|_2^2  \\
 & + \theta\sum\limits_{i=1}^{n_d} \sum\limits_{j=1}^{n_d} \mat{A_{ij}}\|\matsup{K}{i}\matsup{w}{i} - \matsup{K}{ij}\matsup{w}{j} \|_2^2\\
 & + \lambda \sum\limits_{i=1}^{n_d} \matsup{w}{i}' \matsup{K}{i}\matsup{w}{i}\\
\end{array}
\label{eq:joint_kernel}
\end{equation}
\noindent where $\theta$ is a balance parameter to control the importance of domain relations, and $\lambda$ is a regularization parameter. From the above objective function we can see that, if the $i$-th domain and $j$-th domain are closely related, i.e., $\mat{A}_{ij}$ is a large positive number, the predicted value of papers in the $i$-th domain computed using training samples from the $i$-th domain ($\matsup{K}{i}\matsup{w}{i}$) should be similar to that using training samples from the $j$-th domain ($\matsup{K}{ij}\matsup{w}{j}$). 

\subsection{iBall -- Closed-form Solutions}
It turns out that both \ourAlgName ~linear and non-linear formulations have the following closed-form solutions (see proof in subsection~\ref{subsec:proof}\hh{fill in}): 

\begin{equation}
\mat{w} = \mat{S}^{-1} \mat{Y}
\label{eq:linear_solution}
\end{equation}

\noindent {\it \ourAlgName ~linear formulation.} In the linear case, we have that
 $\mat{w} = [\matsup{w}{1};\ldots; \matsup{w}{n_d}]$, $\mat{Y} = [\matsup{X}{1}'\matsup{Y}{1};\ldots;\matsup{X}{n_d}'\matsup{Y}{n_d}]$, and $\mat{S}$ is a block matrix composed of $n_d \times n_d$ blocks, each of size $d \times d$, where $d$ is the feature dimensionality. $\mat{S}$ can be computed as follows:
\vspace{-1pt}
\begin{equation}
\begin{blockarray}{c c c c}
      & \matindex{i-th block column} & \matindex{j-th block column} & \\
    \begin{block}{[c c c]c}
       \ldots &  \ldots & \ldots &  \\
       \ldots & \matsup{X}{i}'\matsup{X}{i} + (\theta\sum\limits_{j=1}^{n_d} \mat{A_{ij}} + \lambda)\mat{I} &  -\theta \mat{A_{ij} I} & \parbox{1.5cm}{\matindex{i-th block}\\ \matindex{row}} \\
      \ldots &  \ldots & \ldots  \\
    \end{block}
  \end{blockarray}
\label{eq:linear_S}
\end{equation}

\hide{

The objective function given in Eq~\ref{eq:joint_kernel} can also be solved in a closed form as follows:

\begin{equation}
\mat{w} = \mat{S}^{-1}\mat{Y}
\label{eq:w_joint_kernel}
\end{equation}
}

\noindent {\it \ourAlgName ~non-linear formulation.} In the non-linear case, we have that $\mat{w} = [\matsup{w}{1};\ldots; \matsup{w}{n_d}]$, $\mat{Y}=[\matsup{Y}{1};\ldots;\matsup{Y}{n_d}]$, and $\mat{S}$ is a block matrix composed of $n_d \times n_d$ blocks with the $(i,j)$-th block of size $n_i \times n_j$, where $n_i$ is the number of training samples in the $i$-th domain. $\mat{S}$ can be computed as follows:

\begin{equation}
\begin{blockarray}{c c c c}
      & \matindex{i-th block column} & \matindex{j-th block column} & \\
    \begin{block}{[c c c]c}
       \ldots &  \ldots & \ldots &  \\
       \ldots & (1 + \theta \sum\limits_{j=1}^{n_d} \mat{A_{ij}}) \matsup{K}{i}+\lambda\mat{I} &  -\theta \mat{A_{ij}}\matsup{K}{ij} & \parbox{1.5cm}{\matindex{i-th block}\\\matindex{row}} \\
      \ldots &  \ldots & \ldots  \\
    \end{block}
\end{blockarray}
\label{eq:kernel_joint_S}
\end{equation}


\subsection{iBall -- Scale-up with Dynamic Update}
The major computation cost for the closed-form solutions lies in the matrix inverse $\mat{S}^{-1}$. In the linear case, the size of $\mat{S}$ is $(dn_d) \times (dn_d)$; and so its computational cost is manageable. However, this is not the case for non-linear closed-form solution since the matrix $\mat{S}$ in Eq.~(\ref{eq:kernel_joint_S}) is of size $n \times n$, where $n=\sum_{i=1}^{n_d} n_i$, which is the number of all the training samples. It would be very expensive to store this dense matrix ($O(n^2)$ space) and to compute its inverse ($O(n^3)$ time); especially when the number of training samples is very large, and the model receives new training examples constantly over time (dynamic update). In this subsection, we devise an efficient algorithm to scale up the non-linear closed-form solution and efficiently update the model to accommodate the new training samples over time. The key of the \ourAlgName ~algorithm is to use the low-rank approximation of the $\mat{S}$ matrix to approximate the original $\mat{S}$ matrix to {\em avoid} the matrix inversion; and at each time step, efficiently update the low-rank approximation itself.

After new papers in all the domains are seen at time step $t+1$, the new $\mat{S}_{t+1}$ computed by Eq.~(\ref{eq:kernel_joint_S}) becomes:

\begin{equation}
\begin{blockarray}{c c c c}
      & \matindex{i-th block column} & \matindex{j-th block column} & \\
    \begin{block}{[c c c]c}
       \ldots &  \ldots & \ldots &  \\
       \ldots & (1 + \theta \sum\limits_{j=1}^{n_d} \mat{A_{ij}}) \matsup{K_{t+1}}{i}+\lambda\mat{I} &  -\theta \mat{A_{ij}}\matsup{K_{t+1}}{ij} & \parbox{1.5cm}{\matindex{i-th block}\\\matindex{row}} \\
      \ldots &  \ldots & \ldots  \\
    \end{block}
\end{blockarray}
\label{eq:new_S}
\end{equation}

\noindent where $\matsup{K_{t+1}}{i}$ is the new within-domain kernel matrix for the $i$-th domain and $\matsup{K_{t+1}}{ij}$ is the new cross domain kernel matrix for the $i$-th and $j$-th domains. The two new kernel matrix can be computed as follows:
\begin{equation}
\matsup{K_{t+1}}{i} = 
\begin{blockarray}{[c c]}
 \matsup{K_t}{i} & (\matsup{k_{t+1}}{i})' \\
 \matsup{k_{t+1}}{i} & \matsup{h_{t+1}}{i}\\
\end{blockarray}
\quad
\matsup{K_{t+1}}{ij} = 
\begin{blockarray}{[c c]}
 \matsup{K_t}{ij} & \matsup{k_{t+1}}{ij_*} \\
 \matsup{k_{t+1}}{i_*j} & \matsup{h_{t+1}}{i_* j_*}\\
\end{blockarray}
\end{equation}
\noindent where $\matsup{k_{t+1}}{i}$ is the matrix characterizing the similarity between new training samples and old training samples and can be computed as: $\matsup{k_{t+1}}{i}(a,b) = k(\matsup{x_{t+1}}{i}(a,:), \matsup{X_t}{i}(b, :))$; $\matsup{h_{t+1}}{i}$ is the similarity matrix among new training samples and can be computed as: $\matsup{h_{t+1}}{i}(a,b) = k(\matsup{x_{t+1}}{i}(a,:),\matsup{x_{t+1}}{i}(b,:))$. $\matsup{k_{t+1}}{i_*j}$ is the matrix characterizing the similarity between new training samples in the $i$-th domain and old training samples in the $j$-th domain and can be computed as: $\matsup{k_{t+1}}{i_*j}(a,b) = k(\matsup{x_{t+1}}{i}(a,:)), \matsup{X_t}{j}(b, :)$. Similarly, $\matsup{k_{t+1}}{ij_*}$ measures the similarity between old training samples in the $i$-th domain and new training samples in the $j$-th domain and can be computed as: $\matsup{k_{t+1}}{ij_*} = k(\matsup{X_t}{i}(a,:), \matsup{x_{t+1}}{j}(b,:))$; $\matsup{h_{t+1}}{i_* j_*}$ is the similarity matrix between new training samples from both $i$-th and $j$-th domains and is computed as: $\matsup{h_{t+1}}{i_* j_*} = k(\matsup{x_{t+1}}{i}(a,:), \matsup{x_{t+1}}{j}(b,:))$.

\hide{
}

Given that $\mat{S}_t$ is a symmetric matrix, we can approximate it using top-$r$ eigen-decomposition as: $\mat{S}_t \approx \mat{U}_t \mat{\Lambda}_t \mat{U}'_t$, where $\mat{U}_t$ is an $n \times r$ orthogonal matrix and $\mat{\Lambda}_t$ is an $r\times r$ diagonal matrix with the largest $r$ eigenvalues of $\mat{S}_t$ on the diagonal. If we can directly update the eigen-decomposition of $\mat{S}_{t+1}$ after seeing the new training samples from all the domains, we can efficiently compute the new model parameters as follows:
\begin{equation}
\begin{array}{rl}
\mat{w}_{t+1} &= \mat{S}_{t+1}^{-1} \mat{Y}_{t+1}\\
			  &= \mat{U}_{t+1}\mat{\Lambda}_{t+1}^{-1} \mat{U}'_{t+1}\mat{Y}_{t+1}\\
\end{array}
\end{equation}

\noindent where $\mat{Y}_{t+1} = [\matsup{Y_t}{1};\matsup{y_{t+1}}{1};\ldots;\matsup{Y_t}{n_d};\matsup{y_{t+1}}{n_d}]$. Here, $\mat{\Lambda}_{t+1}^{-1}$ a $r \times r$ diagonal matrix, whose diagonal entries are the reciprocals of the corresponding eigenvalues of $\mat{\Lambda}_{t+1}$. In this way, we avoid the computationally costly matrix inverse in the closed-form solution.

Compare $\mat{S}_{t+1}$ with $\mat{S}_t$, we find that $\mat{S}_{t+1}$ can be obtained by inserting into $\mat{S}_t$ at the right positions with some rows and columns of the kernel matrices involving new training samples, i.e.,$\matsup{k_{t+1}}{i}$, $\matsup{h_{t+1}}{i}$,$\matsup{k_{t+1}}{i_*j}$,$\matsup{k_{t+1}}{ij_*}$,$\matsup{k_{t+1}}{i_*j_*}$ . From this perspective, $\mat{S}_{t+1}$ can be seen as the sum of the following two matrices:
\begin{equation}
\begin{array}{l}

\underbrace{\begin{blockarray}{c c c c}
      & \matindex{i-th block column} & \matindex{j-th block column} & \\
      \begin{block}{[c c c]c}
       \ldots &  \ldots & \ldots &  \\
       \ldots & \begin{bmatrix}
       			\alpha_i\matsup{K_t}{i} & \mat{0}\\
       			\mat{0} & \mat{0}\\
				\end{bmatrix}        &  \begin{bmatrix}
										-\theta \mat{A_{ij}}\matsup{K_t}{ij} & \mat{0}\\
										\mat{0}  & \mat{0}\\
										\end{bmatrix}				 & \parbox{1.5cm}{\matindex{i-th block}\\\matindex{row}} \\
      \ldots &  \ldots & \ldots  \\
    \end{block}
\end{blockarray}}_{\tilde{\mat{S}}_t}\\
+{\tiny
\underbrace{\begin{blockarray}{c c c c}
      & \matindex{i-th block column} & \matindex{j-th block column} & \\
      \begin{block}{[c c c]c}
       \ldots &  \ldots & \ldots &  \\
       \ldots & \begin{bmatrix}
       			\mat{0}& \alpha_i(\matsup{k_{t+1}}{i})'\\
       			\alpha_i\matsup{k_{t+1}}{i} & \alpha_i \matsup{h_{t+1}}{i} + \lambda\mat{I}\\
				\end{bmatrix}        &  \begin{bmatrix}
										\mat{0} & -\theta \mat{A_{ij}}\matsup{k_{t+1}}{ij_*}\\
										-\theta \mat{A_{ij}}\matsup{k_{t+1}}{i_*j}  & -\theta \mat{A_{ij}}\matsup{h_{t+1}}{i_*j_*}\\
										\end{bmatrix}				 & \parbox{1.0cm}{\matindex{i-th block}\\\matindex{row}} \\
      \ldots &  \ldots & \ldots  \\
    \end{block}
\end{blockarray}}_{\Delta \mat{S}}}\\
\myeq \tilde{\mat{S}}_t + \Delta \mat{S}
\end{array}
\end{equation}

\noindent where we denote $1+\theta\sum_{j=1}^{n_d} \mat{A}_{ij}$ by $\alpha_i$. The top-$r$ eigen-decomposition of $\tilde{\mat{S}}_t$ can be directly written out from that of $\mat{S}_t$ as: $\tilde{\mat{S}}_t \approx \tilde{\mat{U}}_t \mat{\Lambda}_t \tilde{\mat{U}}'_t$, where $\tilde{\mat{U}}_t$ can be obtained by inserting into $\mat{U}_t$  corresponding rows of 0, the same row positions as we insert into $\mat{S}_t$ the new kernel matrices. We propose Algorithm~\ref{alg:eigen_update} to update the eigen-decomposition of $\mat{S}_{t+1}$, based on the observation that $\mat{S}_{t+1}$ can be viewed as $\tilde{\mat{S}}_t$ perturbed by a low-rank matrix $\Delta \mat{S}$.
In line 5 of Algorithm~\ref{alg:eigen_update}, the only difference between the partial QR decomposition and the standard one, is that since  $\tilde{\mat{U}}_t$ is already orthogonal, we only need to perform the Gram-Schmidt procedure starting from the first column of $\mat{P}$.

\vspace{-5pt}
\begin{algorithm}[!htb]
\caption{Eigen update of $\mat{S}_{t+1}$}\label{alg:eigen_update}

\KwIn{(1)eigen pair of $\mat{S}_t$: $\mat{U}_t$, $\mat{\Lambda}_{t}$\;
	  (2)feature matrices of new papers in each domain: $\matsup{x_{t+1}}{i},i=1,\ldots, n_d$\;
	  (3)adjacency matrix of domain relation graph $\mat{A}$ \;
	  (4)balance parameters $\theta, \lambda$ }
\KwOut{eigen pair of $\mat{S}_{t+1}$: $\mat{U}_{t+1}$, $\mat{\Lambda}_{t+1}$}
\BlankLine
Obtain $\tilde{\mat{U}}_t$ by inserting into $\mat{U}_t$ rows of 0 at the right positions \;
Compute $\matsup{k_{t+1}}{i}$, $\matsup{h_{t+1}}{i}$,$\matsup{k_{t+1}}{i_*j}$,$\matsup{k_{t+1}}{ij_*}$,$\matsup{k_{t+1}}{i_*j_*}$ for $i=1,\ldots,n_d, j=1,\ldots, n_d$ \;
Construct sparse matrix $\Delta \mat{S}$ \;
Perform eigen decomposition of $\Delta \mat{S}$: $\Delta \mat{S} = \mat{P\Sigma P'}$\;
Perform partial QR decomposition of $[\tilde{\mat{U}}_t, \mat{P}]$:$[\tilde{\mat{U}}_t, \Delta\mat{Q}]\mat{R} \leftarrow \mbox{QR}(\tilde{\mat{U}}_t,\mat{P})$\;
Set $\mat{Z} = \mat{R}[\mat{\Lambda}_t ~ \mat{0}; \mat{0} ~ \mat{\Sigma}]\mat{R}'$\;
Perform full eigen decomposition of $\mat{Z}$: $\mat{Z} = \mat{V} \mat{L}\mat{V}'$\;
Set $\mat{U}_{t+1} = [\tilde{\mat{U}}_{t}, \Delta \mat{Q}]\mat{V}$ and $\mat{\Lambda}_{t+1} = \mat{L}$\;
{\bf Return}: $\mat{U}_{t+1}$, $\mat{\Lambda}_{t+1}$.
\end{algorithm}

Building upon Algorithm~\ref{alg:eigen_update}, we have the fast \ourAlgName ~algorithm (Algorithm~\ref{alg:approx_kernel_joint}) for scaling up the non-linear solution with dynamic model update.
\vspace{-5pt}
\begin{algorithm}[!htb]
\caption{\ourAlgName ~--scale-up with dynamic update}\label{alg:approx_kernel_joint}

\KwIn{(1)eigen pair of $\mat{S}_t$: $\mat{U}_t$, $\mat{\Lambda}_{t}$\;
	  (2)feature matrices of new papers in each domain: $\matsup{x_{t+1}}{i},i=1,\ldots, n_d$\;
	  (3)citation count vectors of new papers in each domain: $\matsup{y_{t+1}}{i},i=1,\ldots, n_d$\;
	  (4)adjacency matrix of domain relation graph $\mat{A}$ \;
	  (5)balance parameters $\theta, \lambda$ }
\KwOut{(1) updated model parameters $\mat{w}_{t+1}$, (2) eigen pair of $\mat{S}_{t+1}$: $\mat{U}_{t+1}$, $\mat{\Lambda}_{t+1}$}
\BlankLine
Update the eigen-decomposition of $\mat{S}_{t+1}$ using Algorithm~\ref{alg:eigen_update} as: $\mat{S}_{t+1} \approx \mat{U}_{t+1} \mat{\Lambda}_{t+1}\mat{U}'_{t+1}$\;
Compute the new model parameters: $\mat{w}_{t+1} = \mat{U}_{t+1}\mat{\Lambda}_{t+1}^{-1} \mat{U}'_{t+1}\mat{Y}_{t+1}$\;
{\bf Return}: $\mat{w}_{t+1}$, $\mat{U}_{t+1}$ and $\mat{\Lambda}_{t+1}$.
\end{algorithm}
\subsection{iBall -- Proofs and Analysis}\label{subsec:proof}
In this subsection, we will provide some analysis regarding the optimality, the approximation quality as well as the computational complexity of our proposed algorithms. 

\hh{(1) merge lemma 1 and 2; and only give the proof for linear case; (2) hide the proof for eigen-update; and (3) hide the proof for complexity}
{\bf A - Correctness of the closed-form solutions of the \ourAlgName ~linear and non-linear formulations:} In Lemma~\ref{lm:closed_form_solution}, we prove that the closed-form solution given in Eq.~(\ref{eq:linear_solution}) with $\mat{S}$ computed by Eq.~(\ref{eq:linear_S}) is the fixed-point solution to the linear formulation in Eq.~(\ref{eq:joint_linear}) and the closed-form solution given in 
Eq.~(\ref{eq:linear_solution}) with $\mat{S}$ computed by Eq.~(\ref{eq:kernel_joint_S}) is the fixed-point solution to the non-linear formulation in Eq.~(\ref{eq:joint_kernel}).

\begin{lemma}\label{lm:closed_form_solution}{(Correctness of closed-form solution of the \ourAlgName ~linear and non-linear formulations.)} 
For the closed-form solution given in Eq.~(\ref{eq:linear_solution}), if $\mat{S}$ is computed by Eq.~(\ref{eq:linear_S}), it is the fixed-point solution to the objective function in Eq.~(\ref{eq:joint_linear}); and if $\mat{S}$ is computed by Eq.~(\ref{eq:kernel_joint_S}), it is the fixed-point solution to the objective function in Eq.~(\ref{eq:joint_kernel}).
\end{lemma}
\begin{proof} 
Let's take the partial derivative of the objective function (denoted by $J$) in Eq.~(\ref{eq:joint_linear}) w.r.t. $ \matsup{w}{i}$, we get
\begin{equation}
\begin{array}{r l}
\frac{\partial J}{\partial \matsup{w}{i}} &= 2 \matsup{X}{i}'\matsup{X}{i}\matsup{w}{i} - 2 \matsup{X}{i}'\mat{Y^{(i)}} \\
			&+ \sum_{j=1}^{n_d} 2\theta \mat{A}_{ij}(\matsup{w}{i} - \matsup{w}{j}) + 2\lambda \matsup{w}{i}\\
\end{array}
\end{equation}
Now, the derivative of $J$ w.r.t. all the parameters $\mat{w}$ can be computed as:

\begin{equation}
\begin{array}{l l}
\frac{\partial J}{\partial \mat{w}} &= \begin{bmatrix}
										\frac{\partial J}{\partial \matsup{w}{1}}\\
										\ldots\\
										\frac{\partial J}{\partial \matsup{w}{n_d}}\\
										\end{bmatrix}\\ 
								   &= \begin{bmatrix}
								   2 \matsup{X}{1}'\matsup{X}{1}\matsup{w}{1} - 2 \matsup{X}{1}'\mat{Y^{(1)}}\\
			+ \sum_{j=1}^{n_d} 2\theta \mat{A}_{1j}(\matsup{w}{1} - \matsup{w}{j}) + 2\lambda \matsup{w}{1}\\
			\ldots\\
			 2 \matsup{X}{n_d}'\matsup{X}{n_d}\matsup{w}{n_d} - 2 \matsup{X}{n_d}'\mat{Y^{(n_d)}}\\
			+ \sum_{j=1}^{n_d} 2\theta \mat{A}_{n_d j}(\matsup{w}{n_d} - \matsup{w}{j}) + 2\lambda \matsup{w}{n_d}\\
								   		\end{bmatrix}\\
\end{array}
\end{equation}

Set the above derivative to 0 and with some rearrangement, we get
\begin{equation}
\mat{S w} = \mat{Y}
\end{equation}
Therefore, $\mat{w} = \mat{S}^{-1} \mat{Y}$.

The similar procedure can be applied to get the closed-form solution for the non-linear formulation. We will omit the derivation for brevity.
\end{proof}

\hide{
{\bf Correctness of the closed-form solution of the \ourAlgName non-linear formulation:} In Lemma~\ref{lm:kernel_solution}, we prove that the closed-form solution given in Eq.~\ref{eq:w_joint_kernel} is the optimal solution to the non-linear formulation in Eq.~\ref{eq:joint_kernel}.

\begin{lemma}\label{lm:kernel_solution}{(Correctness of the closed-form solution of the \ourAlgName non-linear formulation.)}
Eq.~\ref{eq:w_joint_kernel} is the optimal solution to the objective function in Eq.~\ref{eq:joint_kernel}.
\end{lemma}
\begin{proof} 
Let's take the partial derivative of the objective function (denoted by $J$) in Eq.~(\ref{eq:joint_kernel}) w.r.t. $\matsup{w}{i}$, we get
\begin{equation}
\begin{array}{r l}
\frac{\partial J}{\partial \matsup{w}{i}} &= 2\matsup{K}{i}'\matsup{K}{i}\matsup{w}{i} - 2\matsup{K}{i}'\matsup{Y}{i}\\
										  &+ \sum\limits_{i=1}^{n_d} 2\theta\mat{A}_{ij}(\matsup{K}{i}'\matsup{K}{i}\matsup{w}{i} -\matsup{K}{i}'\matsup{K}{ij}\matsup{w}{j} ) + 2\lambda \matsup{K}{i}\matsup{w}{i}\\
\end{array}
\end{equation}

Now, the derivative of $J$ w.r.t. all the parameters $\mat{w}$ can be computed as:

\begin{equation}
\begin{array}{l }
\frac{\partial J}{\partial \mat{w}} = \begin{bmatrix}
										\frac{\partial J}{\partial \matsup{w}{1}}\\
										\ldots\\
										\frac{\partial J}{\partial \matsup{w}{n_d}}\\
										\end{bmatrix}\\ 
								   = {\tiny\begin{bmatrix}
								   2\matsup{K}{1}'\matsup{K}{1}\matsup{w}{1} - 2\matsup{K}{1}'\matsup{Y}{1}\\
										  + \sum\limits_{i=1}^{n_d} 2\theta\mat{A}_{1j}(\matsup{K}{1}'\matsup{K}{1}\matsup{w}{1} -\matsup{K}{1}'\matsup{K}{1j}\matsup{w}{j} ) + 2\lambda \matsup{K}{1}\matsup{w}{1}\\
			\ldots\\
			 2\matsup{K}{n_d}'\matsup{K}{n_d}\matsup{w}{n_d} - 2\matsup{K}{n_d}'\matsup{Y}{n_d}\\
										  + \sum\limits_{i=1}^{n_d} 2\theta\mat{A}_{n_dj}(\matsup{K}{n_d}'\matsup{K}{n_d}\matsup{w}{n_d} -\matsup{K}{n_d}'\matsup{K}{n_dj}\matsup{w}{j} ) + 2\lambda \matsup{K}{n_d}\matsup{w}{n_d}\\
								   		\end{bmatrix}}\\
\end{array}
\end{equation}
Set the above derivative to 0 and with some rearrangement, we get
\begin{equation}
\mat{Sw} = \mat{Y}
\end{equation}
Therefore, $\mat{w} = \mat{S}^{-1} \mat{Y}$.
\end{proof}
}

{\bf B - Correctness of the eigen update of $\mat{S}_{t+1}$:}
The critical part of Algorithm~\ref{alg:approx_kernel_joint} is the subroutine Algorithm~\ref{alg:eigen_update} for updating the eigen-decomposition of $\mat{S}_{t+1}$. According to Lemma~\ref{lm:eigen_update}, the only place that approximation error occurs is the initial eigen-decomposition of $\mat{S}_0$. The eigen updating procedure won't introduce additional error.

\begin{lemma}{(Correctness of Algorithm~\ref{alg:eigen_update}.)}\label{lm:eigen_update}
If $\mat{S}_t = \mat{U}_t\mat{\Lambda}_t\mat{U}'_t$ holds,  Algorithm~\ref{alg:eigen_update} gives the exact eigen-decomposition of $\mat{S}_{t+1}$.
\end{lemma}
\begin{proof}
Omitted for brevity. See~\cite{conf/sdm/LiTXF15} for details.
\hide{
If $\mat{S}_t = \mat{U}_t\mat{\Lambda}_t\mat{U}'_t$ holds, then $\tilde{\mat{S}}_t = \tilde{\mat{U}}_t \mat{\Lambda}_t \mat{U0}'_t$ also holds. Since $\Delta \mat{S}$ is a symmetric matrix, we can again write its eigen-decomposition as follows:
\begin{equation}
\Delta \mat{S} = \mat{P\Sigma P'}
\end{equation}
\noindent where $\mat{P}$ and $\mat{\Sigma}$ are the eigen pair of $\Delta \mat{S}$. After the update, we can write $\mat{S}_{t+1}$ as the sum of $\tilde{\mat{S}}_t$ and $\Delta \mat{S}$ as follows:
\begin{equation}
\begin{array}{rl}
\mat{S}_{t+1} &= \tilde{\mat{S}}_t + \Delta \mat{S} \\
			  &= \tilde{\mat{U}}_t \mat{\Lambda}_t \mat{U0}'_t + \mat{P\Sigma P'}\\
			  &= \begin{bmatrix}
			  \tilde{\mat{U}}_t & \mat{P}\\
			  \end{bmatrix}
			  \begin{bmatrix}
			  \mat{\Lambda}_t & \mat{0}\\
			  \mat{0} & \mat{\Sigma}\\
			  \end{bmatrix}
			  \begin{bmatrix}
			  \tilde{\mat{U}}_t & \mat{P}\\
			  \end{bmatrix}'\\
\end{array}
\end{equation}
Denote $[\tilde{\mat{U}}_t~\mat{P}]$ by $\tilde{\mat{U}}$ and perform a decomposition on  $\tilde{\mat{U}}$ similar to QR decomposition and get the following:
\begin{equation}
\tilde{\mat{U}} = \begin{bmatrix}
				\tilde{\mat{U}}_t & \Delta \mat{Q}\\
				 \end{bmatrix}
				 \begin{bmatrix}
				 \mat{I} & \mat{R}_1\\
				 \mat{0} & \mat{R}_2\\
				 \end{bmatrix}
\end{equation}
where $[\tilde{\mat{U}}_t ~ \Delta \mat{Q}]$ is orthogonal and $[\mat{I} ~ \mat{R}_1; \mat{0}~\mat{R}_2]$ is an upper triangle matrix. The difference between this partial QR decomposition and standard QR decomposition is that since  $\tilde{\mat{U}}_t$ is already orthogonal as it is the eigenvectors of $\tilde{\mat{S}}_t$, we only need to perform the Gram-Schmidt procedure starting from the first column of $\mat{P}$. Now, let's further expand $\mat{S}_{t+1}$ as follows:
\begin{equation}
\begin{array}{rl}
\mat{S}_{t+1} =& \tilde{\mat{U}}\begin{bmatrix}
			  \mat{\Lambda}_t & \mat{0}\\
			  \mat{0} & \mat{\Sigma}\\
			  \end{bmatrix} \tilde{\mat{U}}'\\
			  =& \begin{bmatrix}
				\tilde{\mat{U}}_t & \Delta \mat{Q}\\
				 \end{bmatrix}
				 \begin{bmatrix}
				 \mat{I} & \mat{R}_1\\
				 \mat{0} & \mat{R}_2\\
				 \end{bmatrix} 
				 \begin{bmatrix}
			  \mat{\Lambda}_t & \mat{0}\\
			  \mat{0} & \mat{\Sigma}\\
			  \end{bmatrix} 				 
				 \begin{bmatrix}
				 \mat{I} & \mat{R}_1\\
				 \mat{0} & \mat{R}_2\\
				 \end{bmatrix}'
				 \begin{bmatrix}
				\tilde{\mat{U}}_t & \Delta \mat{Q}\\
				 \end{bmatrix}'
\end{array}
\end{equation}

Denote $\begin{bmatrix}
				 \mat{I} & \mat{R}_1\\
				 \mat{0} & \mat{R}_2\\
				 \end{bmatrix} 
				 \begin{bmatrix}
			  \mat{\Lambda}_t & \mat{0}\\
			  \mat{0} & \mat{\Sigma}\\
			  \end{bmatrix} 				 
				 \begin{bmatrix}
				 \mat{I} & \mat{R}_1\\
				 \mat{0} & \mat{R}_2\\
				 \end{bmatrix}'$ by $\mat{Z}$ 
and perform a full eigen decomposition of $\mat{Z}$ as $\mat{Z} = \mat{VLV'}$, where $\mat{V}$ and $\mat{L}$ are its eigen pairs. Finally, $\mat{S}_{t+1}$ can be written as the following:
\begin{equation}
\begin{array}{rl}
\mat{S}_{t+1} &= \begin{bmatrix}
				\tilde{\mat{U}}_t & \Delta \mat{Q}\\
				 \end{bmatrix} \mat{VLV'}  \begin{bmatrix}
				\tilde{\mat{U}}_t & \Delta \mat{Q}\\
				 \end{bmatrix}'\\
			  &\myeq \mat{U}_{t+1} \mat{\Lambda}_{t+1} \mat{U}'_{t+1}\\
\end{array}
\end{equation}
\noindent $\mat{U}_{t+1} = [\tilde{\mat{U}}_t ~ \Delta \mat{Q}] \mat{V}$ is the updated eigenvectors of $\mat{S}_{t+1}$ and $\mat{\Lambda}_{t+1}=\mat{L}$ is the updated eigenvalues of  $\mat{S}_{t+1}$.

Summarizing the above derivations, we have the eigen update algorithm in Algorithm~\ref{alg:eigen_update}.
}
\end{proof}

{\bf C - Approximation Quality:}
We analyze the approximation quality of Algorithm~\ref{alg:approx_kernel_joint} to see how much the learned model parameters deviate from the parameters learned using the exact \ourAlgName ~non-linear formulation. The result is summarized in Theorem ~\ref{theorem:error_bound}.

\begin{theorem}{(Error bound of Algorithm~\ref{alg:approx_kernel_joint}.)}\label{theorem:error_bound}
In Algorithm~\ref{alg:approx_kernel_joint}, if $\frac{\sum_{i\notin \mathcal{H}} \lambda_t^{(i)}}{\sum_i \lambda_{t+1} ^{(i)}} <1$, the error of the learned model parameters is bounded by:
\begin{equation}
\|\mat{w}_{t+1} - \hat{\mat{w}}_{t+1}\|_2 \leq  \frac{\sum_{i\notin \mathcal{H}} \lambda_t^{(i)}}{(\sum_i\lambda_{t+1}^{(i)})^2 (1-\delta)} \|\mat{Y}_{t+1}\|_2
\end{equation}
\noindent where $\mat{w}_{t+1} $ is the model parameter learned by the exact \ourAlgName ~non-linear formulation at time $t+1$, $\hat{\mat{w}}_{t+1}$ is the updated model parameter output by Algorithm~\ref{alg:approx_kernel_joint} from time $t$ to $t+1$, 
$\lambda_{t}^{(i)}$ and $\lambda_{t+1}^{(i)}$ are the largest $i$-th eigenvalues of $\mat{S}_{t}$ and $\mat{S}_{t+1}$ respectively, $\delta = \|(\tilde{\mat{U}}_t\mat{\Lambda}_t\tilde{\mat{U}}'_t + \Delta \mat{S})^{-1} (\tilde{\mat{S}}_{t} - \tilde{\mat{U}}_t\mat{\Lambda}_t\tilde{\mat{U}}'_t)\|_F$, $\mathcal{H}$ is the set of integers between 1 and $r$, i.e., $\mathcal{H}=\{a|a\in [1,r]\}$.
\end{theorem}
\begin{proof}
Suppose we know the exact $\mat{S}_t$ at time $t$ and its top-$r$ approximation: $\hat{\mat{S}}_t = \mat{U}_t\mat{\Lambda}_t\mat{U}'_t$. After one time step, we can construct $\Delta \mat{S}$ and the exact $\mat{S}_{t+1}$ can be computed as $\mat{S}_{t+1}  = \tilde{\mat{S}}_{t} + \Delta \mat{S}$. The model parameters learned by the exact non-linear model is:
\begin{equation}
\begin{array}{rl}
\mat{w}_{t+1} &= \mat{S}_{t+1}^{-1} \mat{Y}_{t+1}\\
			  &= (\tilde{\mat{S}}_{t} + \Delta \mat{S})^{-1} \mat{Y}_{t+1}
\end{array}
\end{equation}

If we allow approximation as in  Algorithm~\ref{alg:approx_kernel_joint}, the approximated model parameter is:
\begin{equation}
\begin{array}{rl}
\hat{\mat{w}}_{t+1} &= \hat{\mat{S}}_{t+1}^{-1} \mat{Y}_{t+1}\\
					&= (\tilde{\mat{U}}_t\mat{\Lambda}_t\tilde{\mat{U}}'_t + \Delta \mat{S})^{-1} \mat{Y}_{t+1}
\end{array}
\end{equation}

Denote $\tilde{\mat{S}}_{t} + \Delta \mat{S}$ by  $\mat{B}$ and $\tilde{\mat{U}}_t\mat{\Lambda}_t\tilde{\mat{U}}'_t + \Delta \mat{S}$ by $\mat{C}$,we have the following:
\begin{equation}
\begin{array}{rl}
\| \mat{B} - \mat{C} \|_F &= \|\tilde{\mat{S}}_{t} - \tilde{\mat{U}}_t\mat{\Lambda}_t\tilde{\mat{U}}'_t\|_F\\
						  &\leq \sum_{i\notin \mathcal{H}} \lambda_t^{(i)}
\end{array}
\end{equation}
\noindent where the last inequality is due to the following fact:
\begin{equation}
\begin{array}{rl}
\| \sum_i a_i \mat{u}_i \mat{u}'_i\|_F &= \sqrt{\mbox{tr}(\sum_i a_i^2\mat{u}_i \mat{u}'_i )}\\
									   &= \sqrt{\sum_i a_i^2 \mbox{tr} (\mat{u}_i \mat{u}'_i)}\\
									   &= \sqrt{\sum_i a_i^2}\\
									   &\leq \sum_i |a_i|
\end{array}
\end{equation}
Denote $\|\mat{C}^{-1} (\mat{B} - \mat{C})\|_F$ by $\delta$, we know that
\begin{equation}
\begin{array}{rl}
\delta &\leq \|\mat{C}^{-1}\|_F\|\mat{B} - \mat{C}\|_F\\
  &\leq \frac{\sum_{i\notin \mathcal{H}} \lambda_t^{(i)}}{\sum_i \lambda_{t+1} ^{(i)}} <1
\end{array}
\end{equation}

From matrix perturbation theory~\cite{Golub:1996:MC:248979}, we will reach the following:
\begin{equation}
\begin{array}{rl}
\|\mat{w}_{t+1} - \hat{\mat{w}}_{t+1}\|_2 &= \|\mat{B}^{-1} \mat{Y}_{t+1} - \mat{C}^{-1} \mat{Y}_{t+1}\|_2\\
										  &\leq \|\mat{B}^{-1} - \mat{C}^{-1}\|_F \|\mat{Y}_{t+1}\|_2\\
										  &\leq \frac{\|\mat{C}^{-1}\|_F^2 \|\mat{B} - \mat{C}\|_F}{1-\delta} \|\mat{Y}_{t+1}\|_2\\
										  &\leq \frac{\sum_{i\notin \mathcal{H}} \lambda_t^{(i)}}{(\sum_i\lambda_{t+1}^{(i)})^2 (1-\delta)} \|\mat{Y}_{t+1}\|_2\\
\end{array}
\end{equation}
\end{proof}

{\bf D - Complexities:}
Finally, we analyze the complexities of Algorithm~\ref{alg:eigen_update} and Algorithm~\ref{alg:approx_kernel_joint}. In terms of time complexity, the savings are two-folds: (1) we only need to compute the kernel matrices involving new training samples; (2) we avoid the time consuming large matrix inverse operation. In terms of space complexity, we don't need to maintain the huge $\mat{S}_t$ matrix, but instead store its top-$r$  eigen pairs which is only of $O(nr)$ space.

\begin{theorem}{(Complexities of Algorithm~\ref{alg:eigen_update} and Algorithm~\ref{alg:approx_kernel_joint}.)}\label{theorem:complexities}
Algorithm~\ref{alg:eigen_update} takes $O((n+m)(r^2 + r'^2))$ time and $O((n+m)(r+r'))$ space. Algorithm~\ref{alg:approx_kernel_joint} also takes $O((n+m)(r^2 + r'^2))$ time and $O((n+m)(r+r'))$ space, where $m$ is total number of new training samples.
\end{theorem}
\begin{proof}
Omitted for brevity.
\hide{
{\it Time complexity of Algorithm~\ref{alg:eigen_update}:} Step 1-3 take $O(nm)$ time, where $n$ is total number of training samples from previous step, and $m$ is the total number of new training samples. Eigen decomposition of $\Delta \mat{S}$ in step 4 takes  $O(nmr')$, where $r'$ is the rank of $\Delta \mat{S}$, since $\Delta \mat{S}$ is sparse matrix with $O(nm)$ non-zero entries. 
QR decomposition in step 5 takes $O((n+m)r'^2)$ since we only need to start from the columns in $\mat{P}$. Step 6 and 7 both take $O((r+r')^3)$ time. The last line takes at most $O((n+m)(r+r')^2)$. The overall time complexity is $O((n+m)(r^2 + r'^2))$.

{\it Space complexity of Algorithm~\ref{alg:eigen_update}:} The storage of eigen pairs requires $O((n+m)r)$ space. Step 1-3 take $O(mn)$ space.  Eigen decomposition of $\Delta \mat{S}$ in step 4 takes $O((n+m)r')$ space. QR decomposition in step 5 needs $O((n+m)(r+r'))$ space. Step 6 and 7 take $O((r+r')^2)$ space and line 8 needs $O((n+m)(r+r'))$. The overall space complexity is $O((n+m)(r+r'))$.

{\it Time complexity of Algorithm~\ref{alg:approx_kernel_joint}:} Update eigen decomposition of $\mat{S}_{t+1}$ in step 1 takes $O((n+m)(r^2 + r'^2))$ time and computing the new learning parameter in step 2 takes $O(n+m)r$ time. The overall time complexity is $O((n+m)(r^2 + r'^2))$.

{\it Space complexity of Algorithm~\ref{alg:approx_kernel_joint}:} Update eigen decomposition of $\mat{S}_{t+1}$ in step 1 takes $O((n+m)(r+r'))$ and computing the new learing parameter in step 2 takes $O((n+m)r)$ space. The overall space complexity is $O((n+m)(r+r'))$.
}
\end{proof}

\section{Experiments}
\label{sec:exp}
In this section, we design and conduct experiments mainly to inspect the following aspects:
\begin{itemize}
\item {\it Effectiveness:} How accurate are the proposed algorithms for predicting scholarly entities' long-term impact?

\item {\it Efficiency:} How fast are the proposed algorithms?
\end{itemize}

\subsection{Experiment Setup}
We use the real-world citation network dataset AMiner\footnote{http://arnetminer.org/billboard/citation} to evaluate our proposed 
algorithms. The statistics and empirical observations are described in Section~\ref{sec:prob}. Our primary task is to predict a paper's citations after 10 years given its citation history in the first three years. Thus, we only keep papers published between year 1936 and 2000 to make sure they are at least 10 years old. This leaves us 508,773 papers. Given that the citation distribution is skewed (see Figure~\ref{fig:citation_dist}\hh{fill in}), the 10-year citation counts are normalized to the range of $[0,7]$. Our algorithm is also able to predict citation counts for other scholarly entities including researchers and venues. We keep authors whose research career (when they publish the first paper) begin between year 1960 and 2000 and venues that are founded before year 2002. This leaves us 315,340 authors and 3,783 venues.

For each scholarly entity, we represent it as a three dimensional feature vector, where the $i$-th dimension is the number of citations the entity receives in the $i$-th year after its life cycle begins (e.g., paper gets published, researchers publish the first paper ). 
We build a $k$-nn graph ($k=5$) among different scholarly entities; use METIS~\cite{Karypis:1998:FHQ:305219.305248} to partition the graph into balanced clusters; and treat each cluster as a domain.  We set the domain number ($n_d$) to be 10 for both papers and researchers; and 5 for venues. The Gaussian kernel matrix of the cluster centroids  is used to construct the domain-domain adjacency matrix $\mat{A}$.

To simulate the dynamic scenario where training samples come in stream, we start with a small initial training set and at each time step add new training samples to it. The training samples in each domain are sorted by starting year (e.g., publication year). In the experiment, for papers, we start with 0.1\% initial training data and at each update add another 0.1\% training samples. The last 10\% samples are reserved as test samples, i.e., we always use information from older publications for the prediction of the latest ones. For authors, we start with 0.2\% initial training data and at each update add another 0.2\% training data and use the last 10\% for testing. For venues, we start with 20\%, add 10\% at each update and use last 10\% for testing.

The root mean squared error (RMSE) between the the actual citation and the predicted one is adopted for accuracy evaluation.  All the experiments were performed on a Windows machine with four 3.5GHz Intel Cores and 256GB RAM. 

{\it Repeatability of Experimenal Results:} The AMiner citation dataset is publicly available. We will release the code of the proposed algorithms through authors' website. For all the results reported in this section, we set $\theta = \lambda = 0.01$ in our joint predictive model. Gaussian kernel with $\sigma = 5.1$ is used in the non-linear formulations.

\subsection{Effectiveness Results}
We perform the effectiveness comparisons of the following nine methods:
\begin{itemize}
\vspace{-1pt}
\setlength\itemsep{0.5pt}
\item [1] {\it Predict 0:} directly predict 0 for test samples since majority of the papers have 0 citations.
\item [2] {\it Sum of the first 3 years:} assume the total number of citations doesn't change after three years.
\item [3] {\it Linear-combine:} combine training samples of all the domains for training using linear regression model.
\item [4] {\it Linear-separate:} train a linear regression model for each domain separately.
\item [5] {\it \ourAlgName-linear:} jointly learn the linear regression models as in our linear formulation.
\item [6] {\it Kernel-combine:} combine training samples of all the domains for training using kernel ridge regression model~\cite{Saunders:1998:RRL:645527.657464}. 
\item [7] {\it Kernel-separate:} train a kernel ridge regression model for each domain separately.
\item [8] {\it \ourAlgName-kernel:} jointly learn the kernel regression models as in our non-linear formulation.
\item [9] {\it \ourAlgName-fast :} proposed algorithm for speeding up the joint non-linear model.
\end{itemize}

\noindent {\it A - Overall paper citation prediction performance.}
The RMSE result of different methods for test samples from all the domains is shown in Figure~\ref{fig:rmse_alldomain}.
We have the following observations: (1) the non-linear methods (\ourAlgName-fast, \ourAlgName-kernel, Kernel-separate, Kernel-combine) outperform the linear methods (\ourAlgName-linear, Linear-separate, Linear-combine) and the straightforward `Sum of first 3 years' is much better than the linear methods, which reflects the complex non-linear relationship between the features and the impact. (2) The performance of \ourAlgName-fast is very close to \ourAlgName-kernel and sometimes even better, which confirms the good approximation quality of the model update and the possible de-noising effect offered by the low-rank approximation. (3) The \ourAlgName~ family of joint models is better than their separate versions (Kernel-separate, Linear-separate). 
To evaluate the statistical significance, we perform a t-test using 1.4\% of the training samples and show the $p$-values in Table~{\ref{tab:ttest}}. From the result, we see that the improvement of our method is significant. To investigate parameter sensitivity, we perform parametric studies with three parameters in \mbox{\ourAlgName-fast}, namely, $\theta$, $\lambda$ and $r$. Figure~{\ref{fig:sen_study}} shows that the proposed method is stable in a large range of the parameter space.

\begin{table*}[!t]
\caption{$p$-value of statistical significance}
\centering
\begin{tabular}{|p{45pt}|p{35pt}|p{35pt}|p{35pt}|p{35pt}|p{40pt}|p{35pt}|p{35pt}|p{35pt}|}
\hline
 & Predict 0 & Linear-combine & Linear-separate & \ourAlgName-linear & Sum of first 3 years & Kernel-combine & Kernel-separate & \ourAlgName-fast \\
 \hline
\ourAlgName-kernel & 0 & 5.53e-16 & 6.12e-17 & 1.16e-13 & 1.56e-219 & 1.60e-72 & 8.22e-30 & 3.39e-14\\
 \hline
\end{tabular}
\label{tab:ttest}
\end{table*}

\begin{figure*}[!htb]
\centering
\minipage[t]{0.32\textwidth}
\centering
 \includegraphics[width = \textwidth]{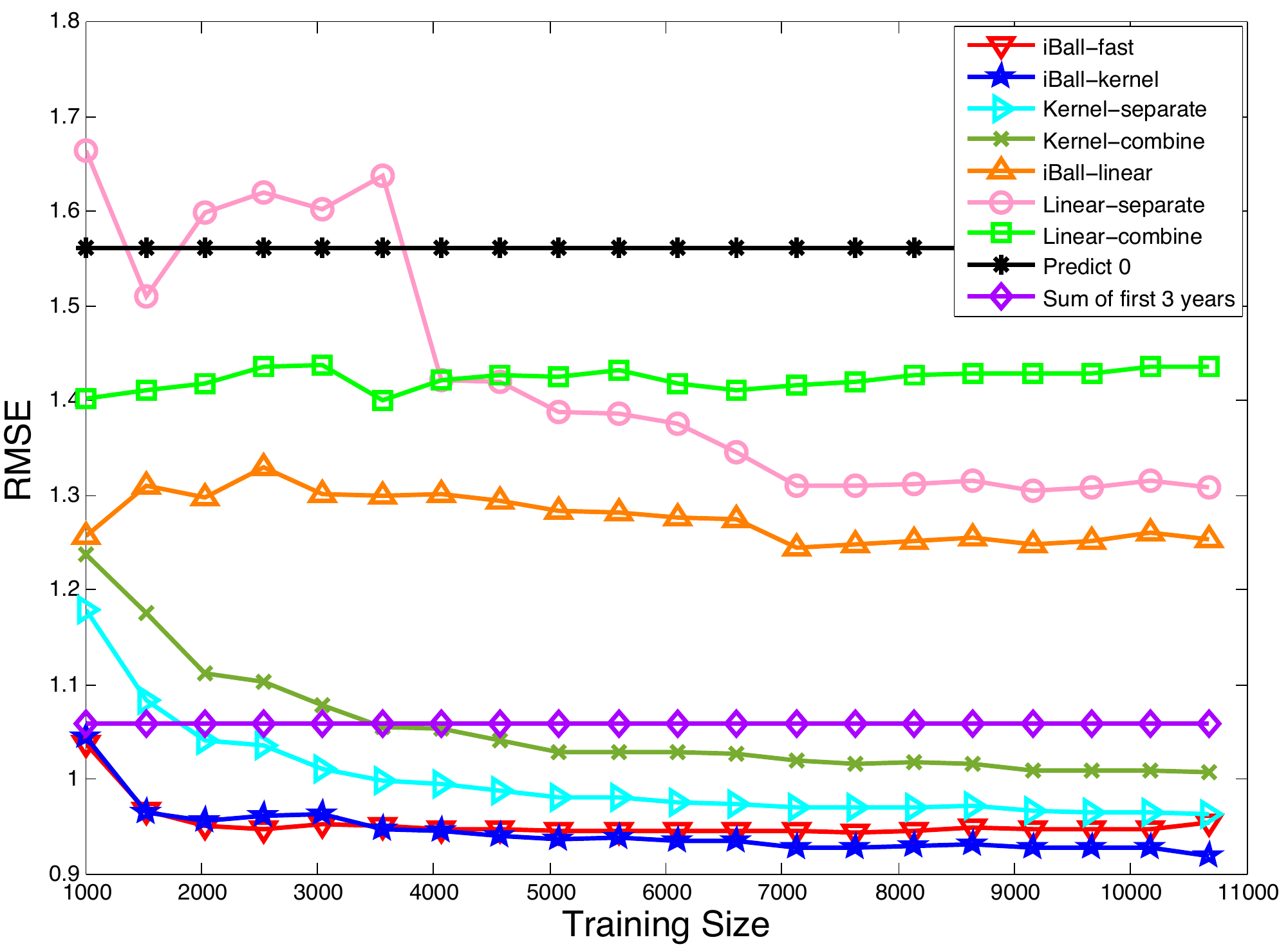}
\caption{Overall paper citation prediction performance comparisons. Lower is better. }\label{fig:rmse_alldomain}
\endminipage\hfill
\minipage[t]{0.32\textwidth}
\centering
 \includegraphics[width = \textwidth]{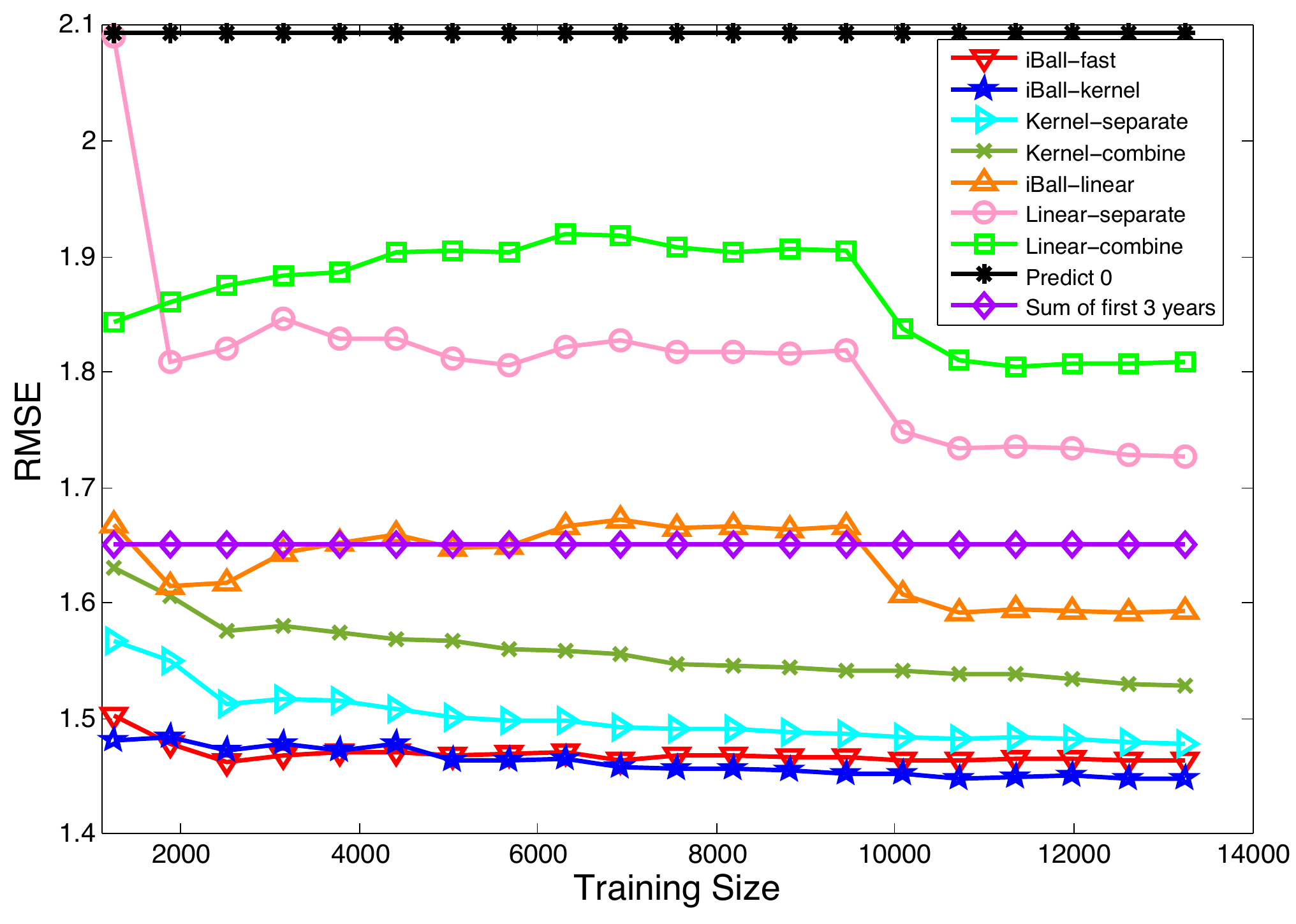}
\caption{Author citation prediction performance comparison. Lower is better. }\label{fig:author_rmse}
\endminipage\hfill
\minipage[t]{0.32\textwidth}%
\centering
 \includegraphics[width = \textwidth]{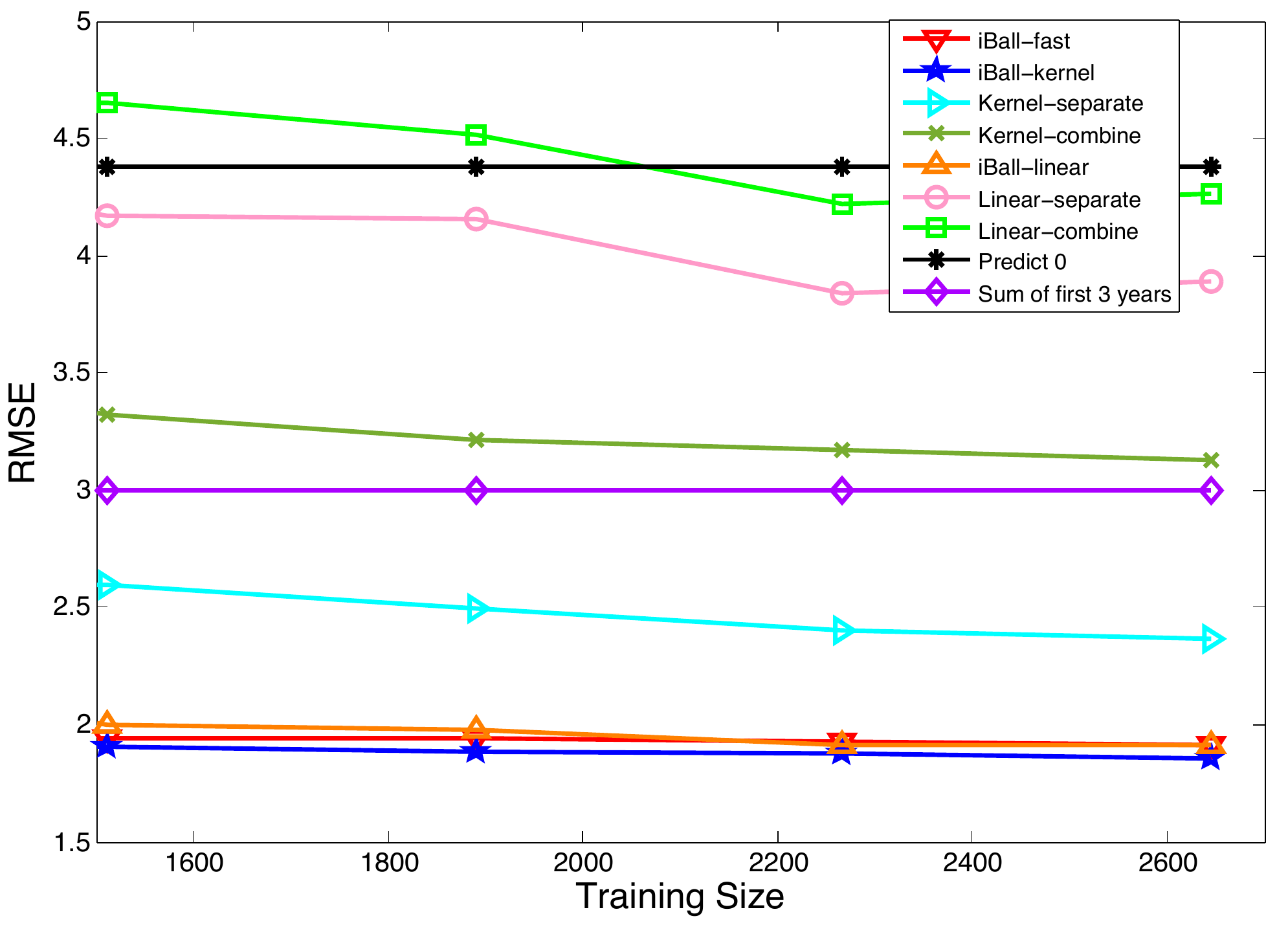}
\caption{Venue citation prediction performance comparison. Lower is better. }\label{fig:venue_rmse}
\endminipage
\vspace{-10pt}
\end{figure*}

\begin{figure*}
    \centering
    \begin{subfigure}[b]{0.3\textwidth}
        \centering
        \includegraphics[width=\textwidth]{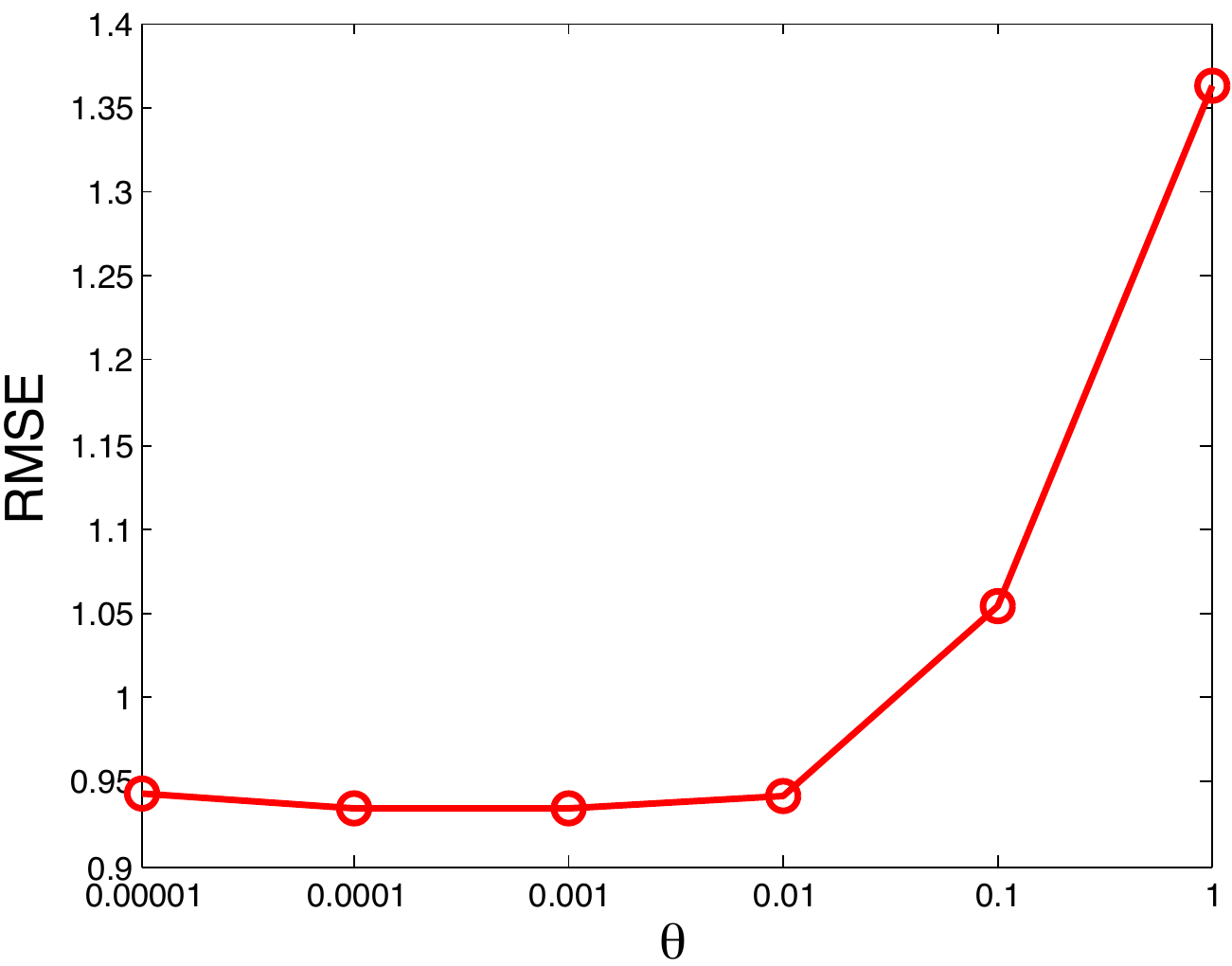}
        \caption{RMSE vs. $\theta$}
    \end{subfigure}
    \hfill
    \begin{subfigure}[b]{0.3\textwidth}
        \centering
        \includegraphics[width=\textwidth]{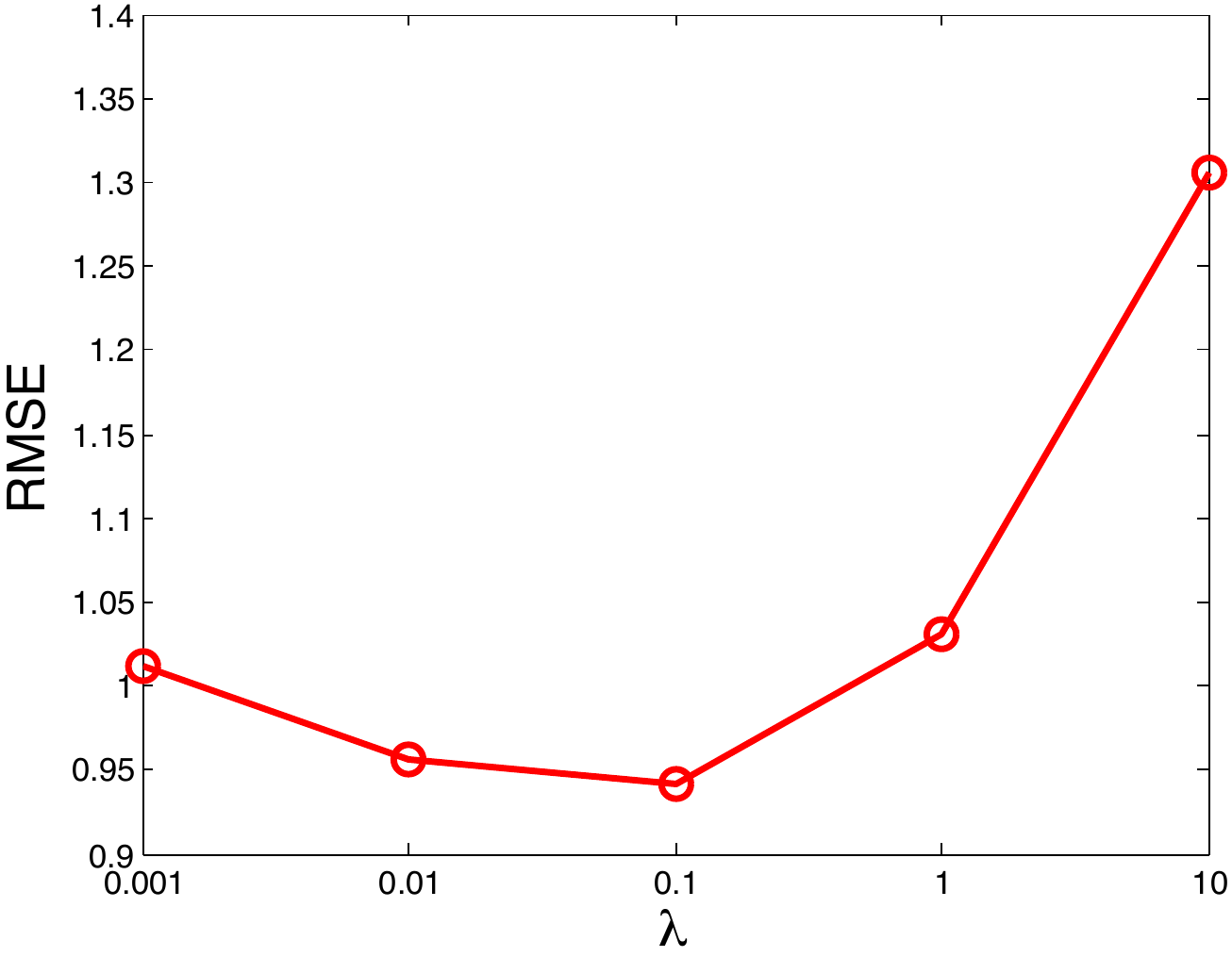}
        \caption{RMSE vs. $\lambda$}
    \end{subfigure}
    \hfill
    \begin{subfigure}[b]{0.3\textwidth}
        \centering
        \includegraphics[width=\textwidth]{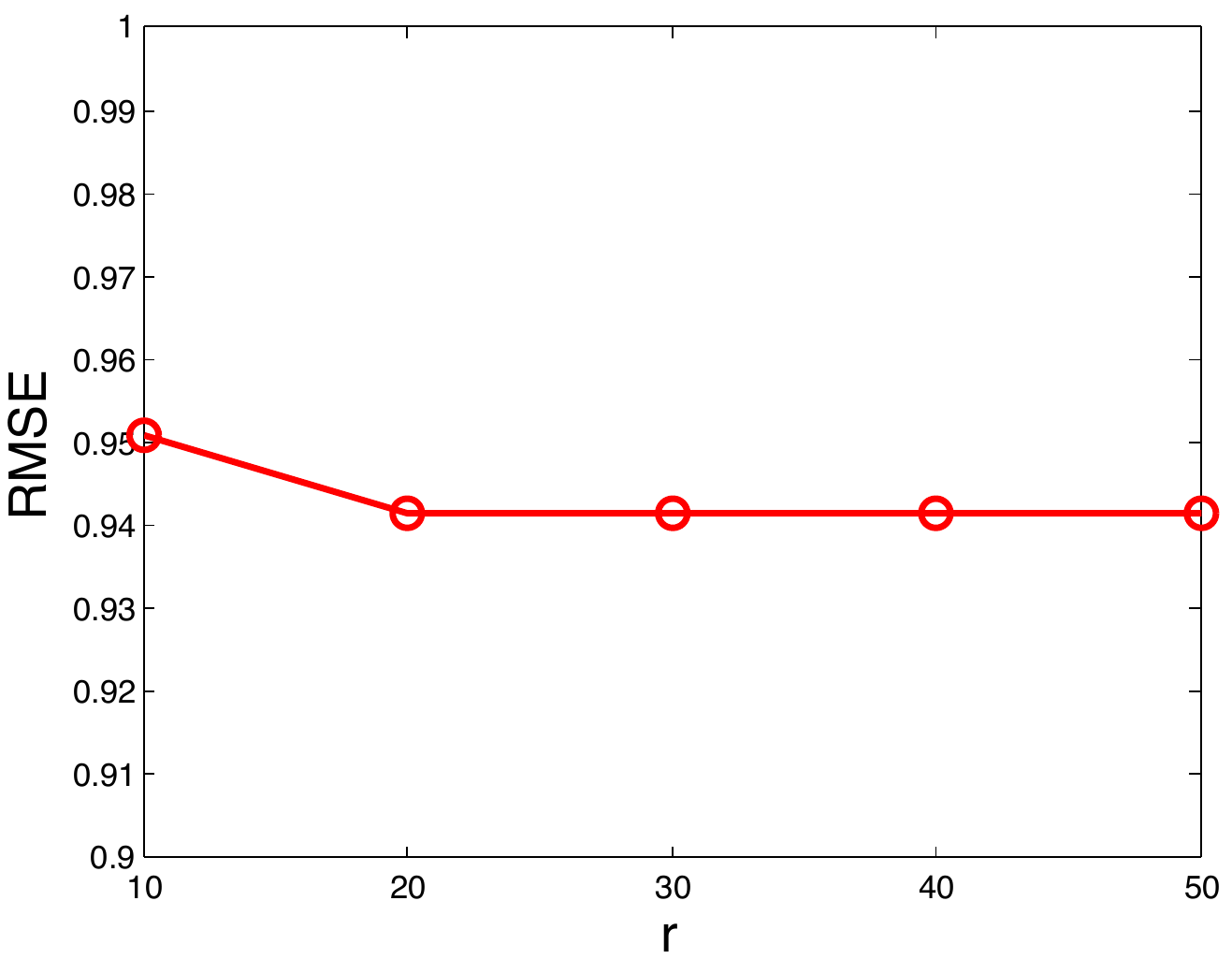}
        \caption{RMSE vs. $r$}
    \end{subfigure}
    \caption{Sensitivity study on \ourAlgName-fast: study the effect of the parameters $\theta$, $\lambda$ and $r$ in terms of RMSE.}
    \label{fig:sen_study}
\end{figure*}


\noindent {\it B -  Domain-by-domain paper citation prediction performance.} In Figure~\ref{fig:domain_rmse} we show the RMSE comparison results for four domains with different total training sizes. \ourAlgName-kernel and its fast version \ourAlgName-fast consistently outperform other methods in all the domains. In the third domain, some linear methods (Linear-separate and Linear-combine) perform even worse than the baseline (`Predict 0').

\begin{figure*}[!htb]
\centering
\begin{subfigure}[t]{0.45\textwidth}
\centering
\includegraphics[width =\textwidth]{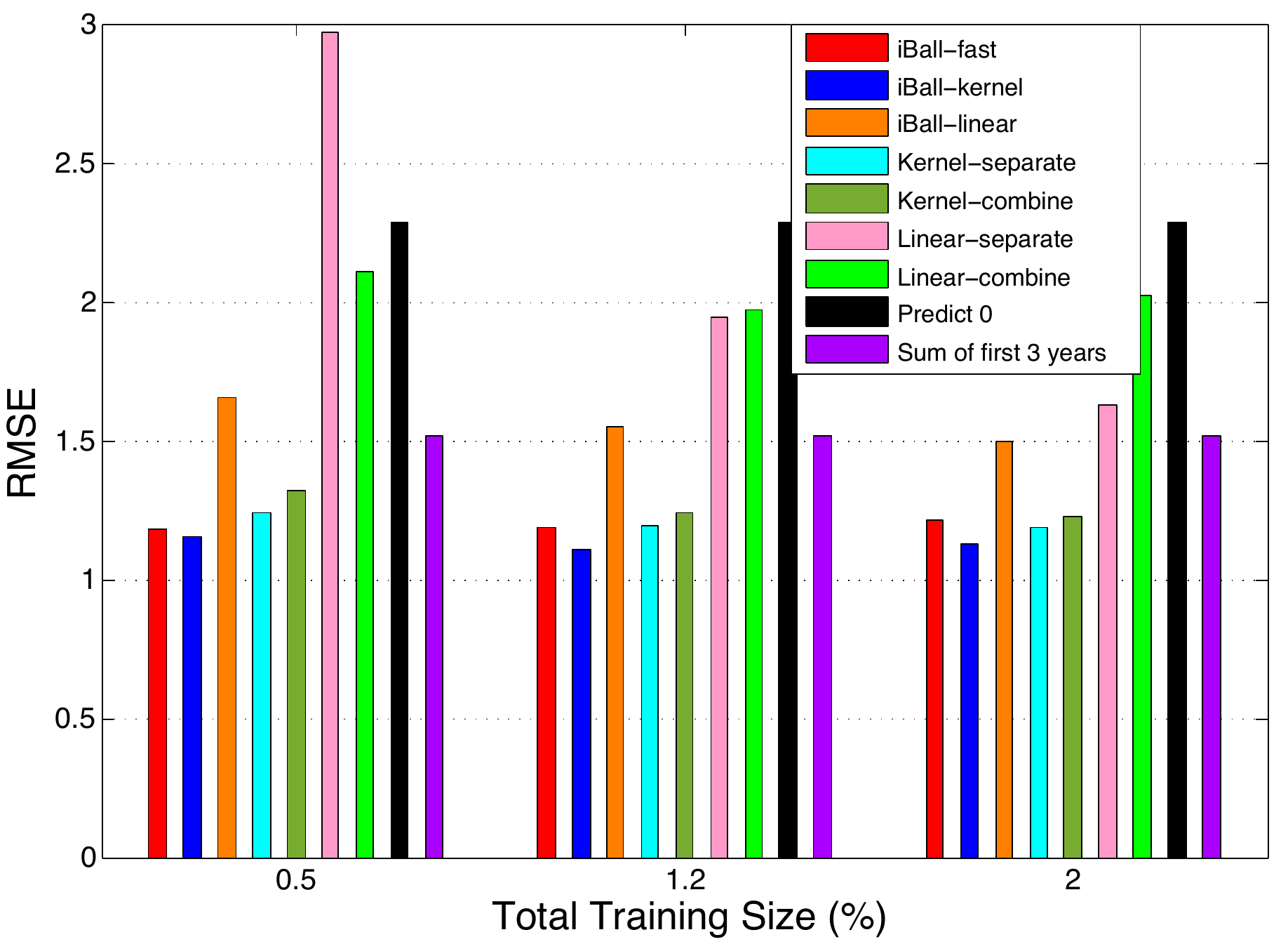}
\caption{Prediction performance comparison in the first domain. }
\end{subfigure}
~
\begin{subfigure}[t]{0.45\textwidth}
\centering
\includegraphics[width =\textwidth]{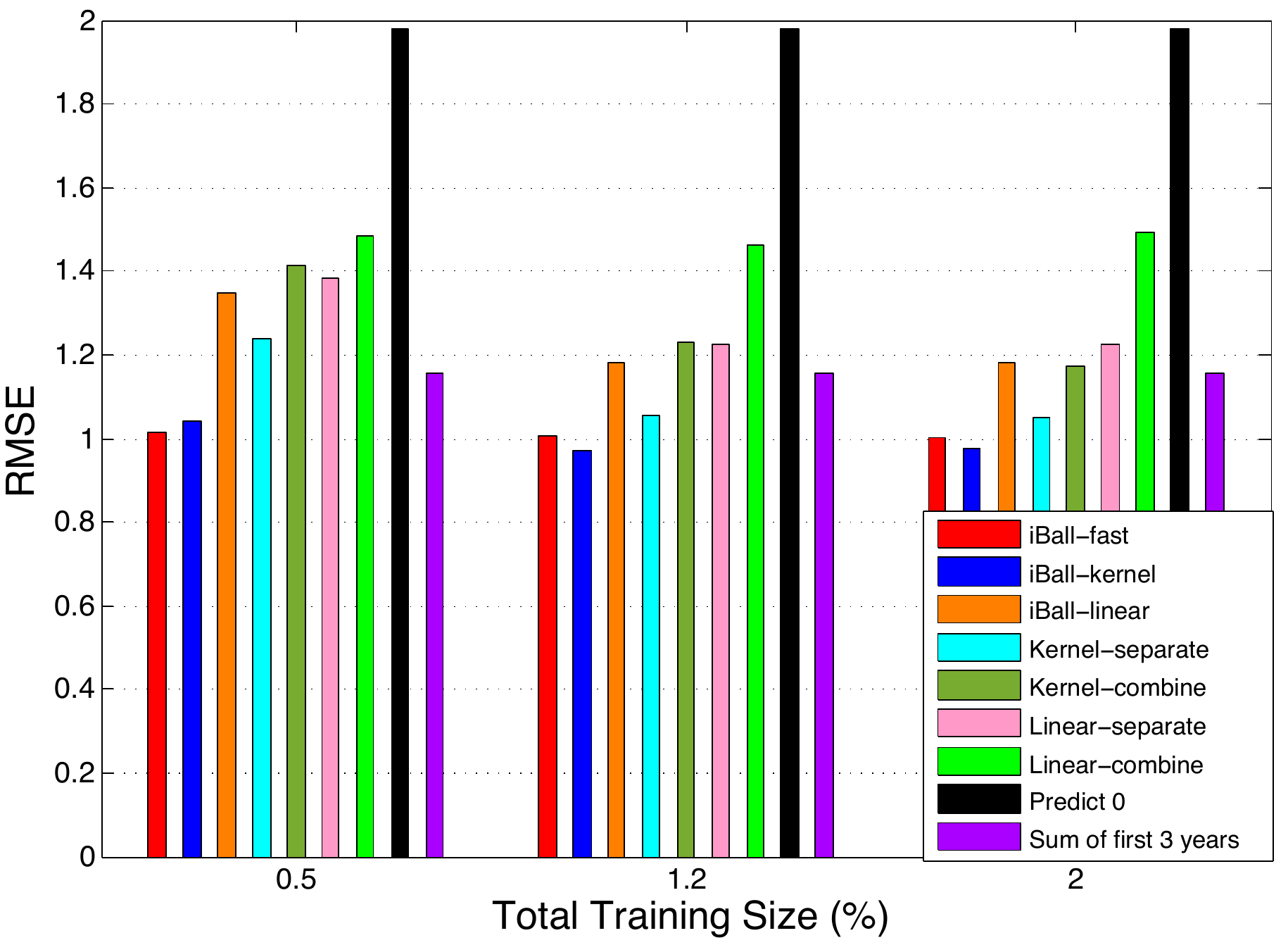}
\caption{Prediction performance comparison  in the second domain. }
\end{subfigure}

\begin{subfigure}[t]{0.45\textwidth}
\centering
\includegraphics[width =\textwidth]{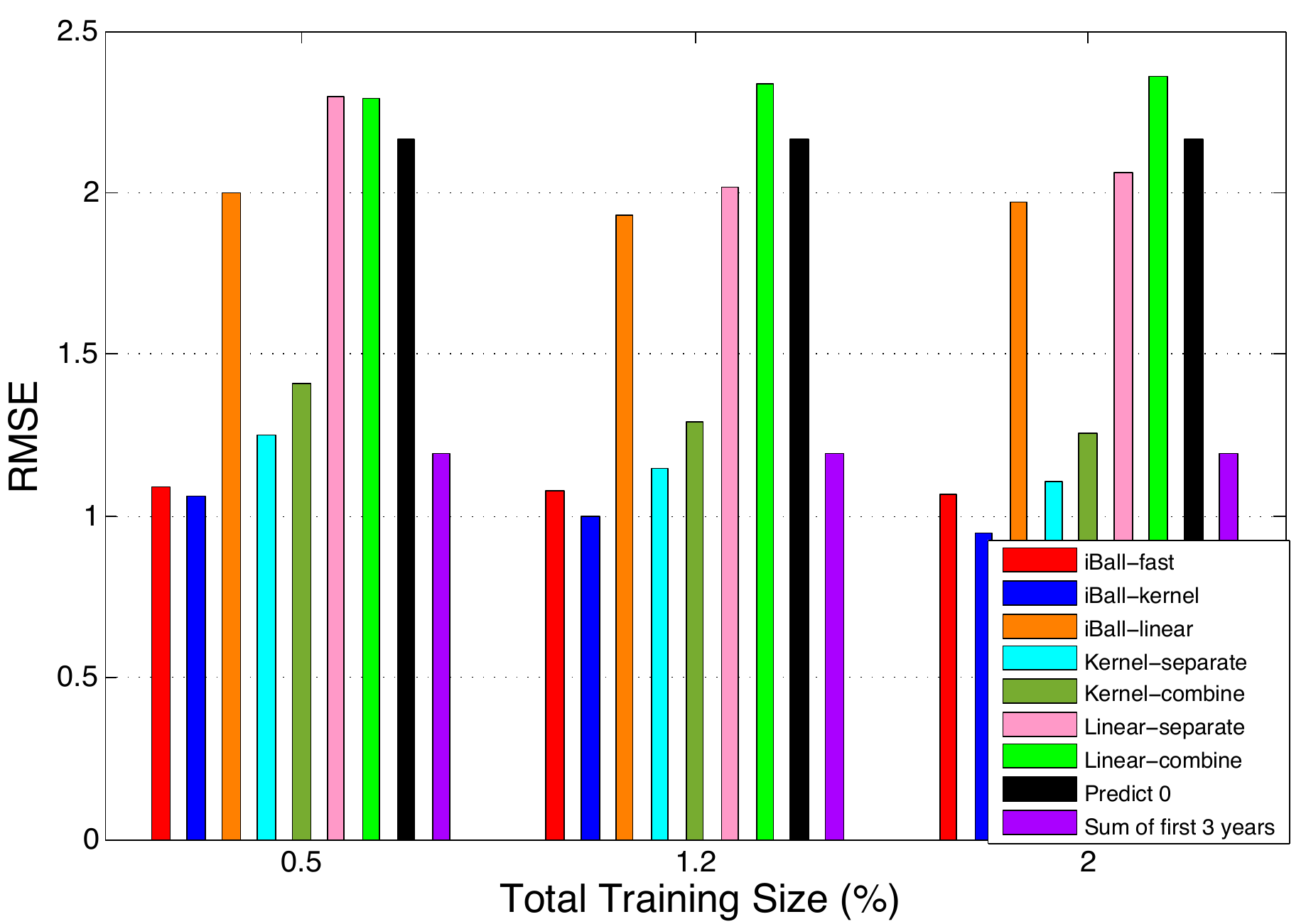}
\caption{Prediction performance comparison in the third domain. }
\end{subfigure}
~
\begin{subfigure}[t]{0.45\textwidth}
\centering
\includegraphics[width = \textwidth]{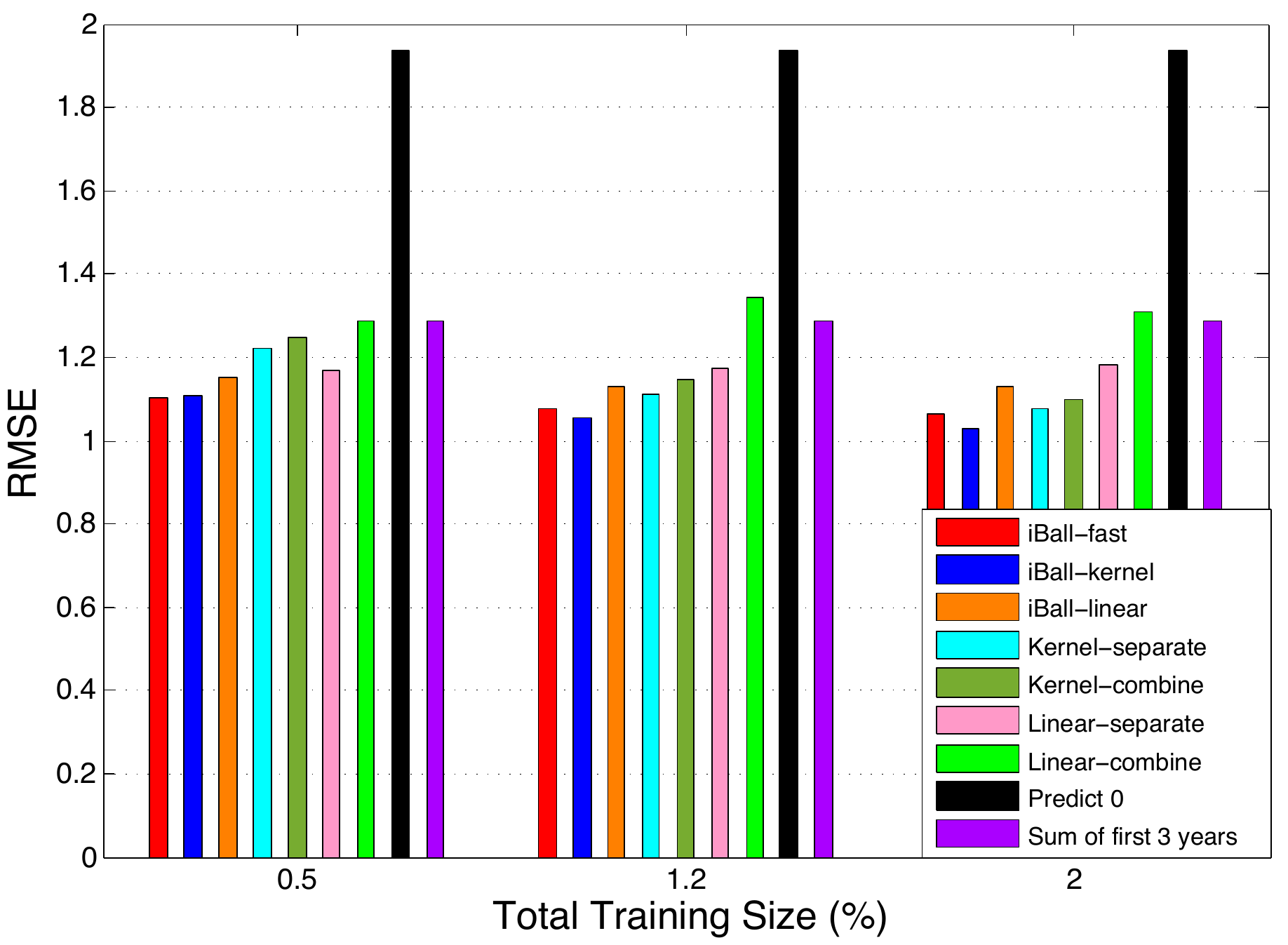}
\caption{Prediction performance comparison in the fourth domain. }
\end{subfigure}
\vspace{-3pt}
\caption{Paper citation prediction performance comparison in four domains.}
\label{fig:domain_rmse}
\vspace{-10pt}
\end{figure*}

\noindent {\it C - Prediction error analysis.}
We visualize the actual citation vs. the predicted citation using \ourAlgName~as a heat map in Figure~\ref{fig:heatmap}. The $(x,y)$ square means among all the test samples with actual citation $y$, the percentage that have predicted citation $x$. We observe a very bright region near the $x=y$ diagonal. The prediction error mainly occurs in a bright strip at $x=1, y\geq 1$. This is probably due to the delayed high-impact of some scientific work, as suggested by the blue and green lines in Figure~\ref{fig:citation_pattern}, i.e., some papers only pick up attentions many years after they were published.

\begin{figure}[!htb]
\centering
\includegraphics[width = 0.45\textwidth]{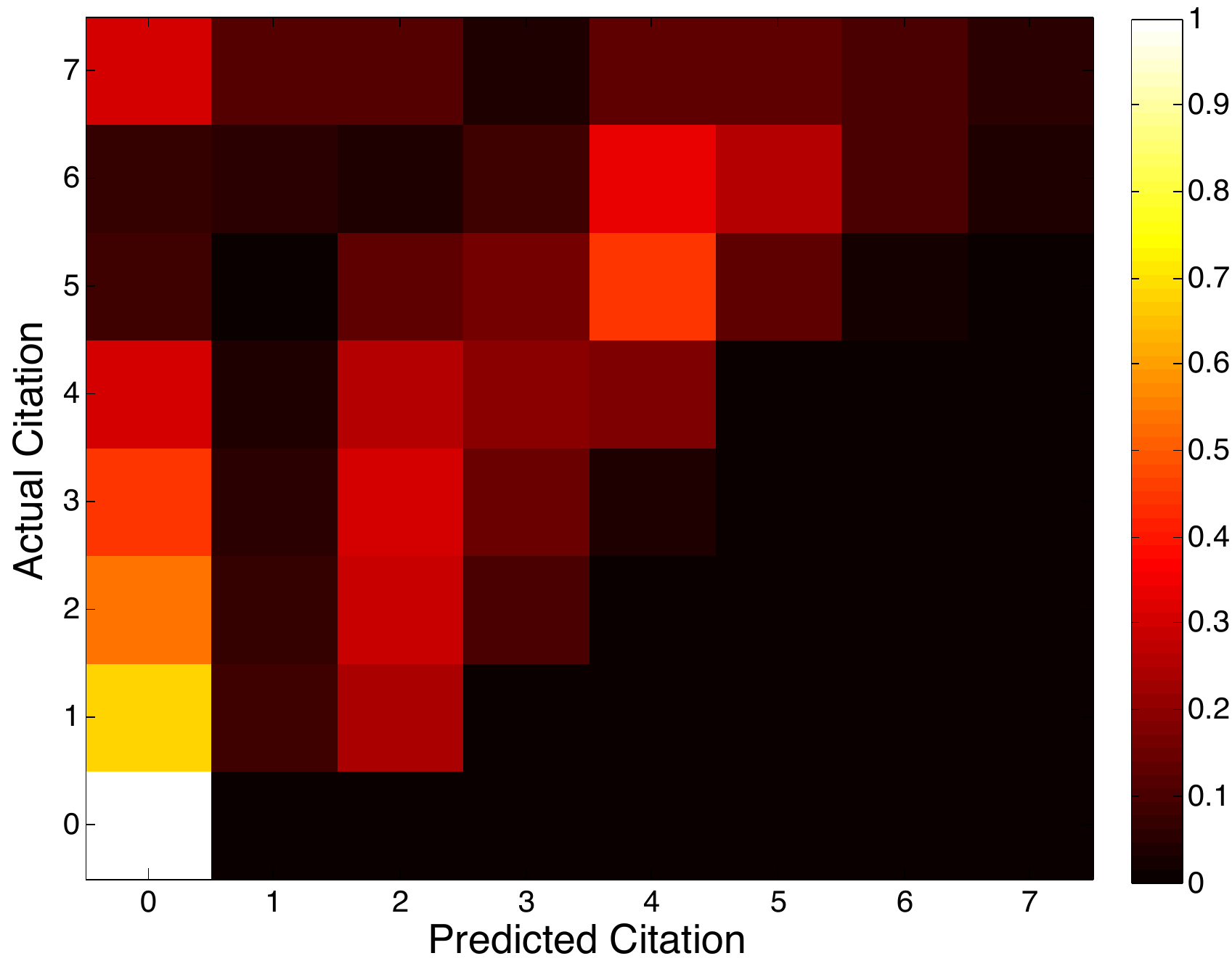}
\caption{Prediction error analysis: actual citation vs. predicted citation. Best viewed in color.}\label{fig:heatmap}
\vspace{-10pt}
\end{figure}

\noindent{\it D - Author and venue citation prediction performance.} We also show the RMSE comparison results for the impact prediction of authors and venues in Figure~\ref{fig:author_rmse} and \ref{fig:venue_rmse} respectively. Similar observations can be made as the paper impact prediction, except that for the venue citation prediction, \ourAlgName-linear can achieve the similar performance as \ourAlgName-fast and \ourAlgName-kernel. This is probably due to the effect that venue citation (which involves the aggregation of the citations of all of its authors and papers) prediction is at a much coarser granularity, and thus a relatively simple linear model is sufficient to characterize the correlation between features and outputs (citation counts).


\subsection{Efficiency Results}

\noindent {\it A - Running time comparison:} We compare the running time of different methods with different training sizes and show the result in Figure~\ref{fig:time_comparison} with time in log scale. All the linear methods are very fast ($<0.01s$) as the feature dimensionality is only 3.  Our \ourAlgName-fast outperforms all other non-linear methods and scales linearly \hh{liangyue: is it true? i thought it is linear}. 

\begin{figure}[!htb]
\minipage[t]{0.22\textwidth}
\includegraphics[width = \textwidth]{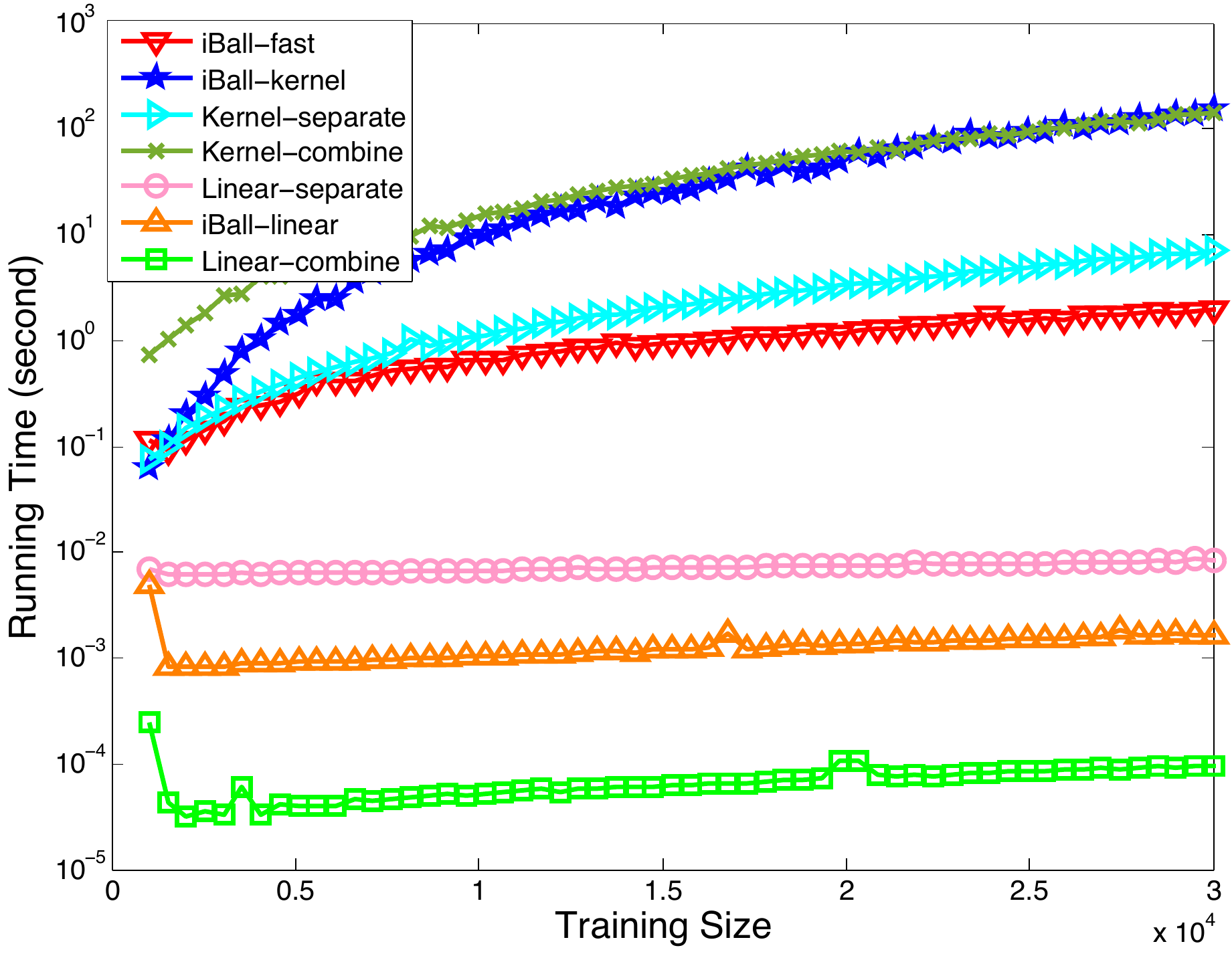}
\caption{Comparison of running time of different methods. The time axis is of log scale. }\label{fig:time_comparison}
\endminipage\hfill
\minipage[t]{0.22\textwidth}
 \includegraphics[width = \textwidth]{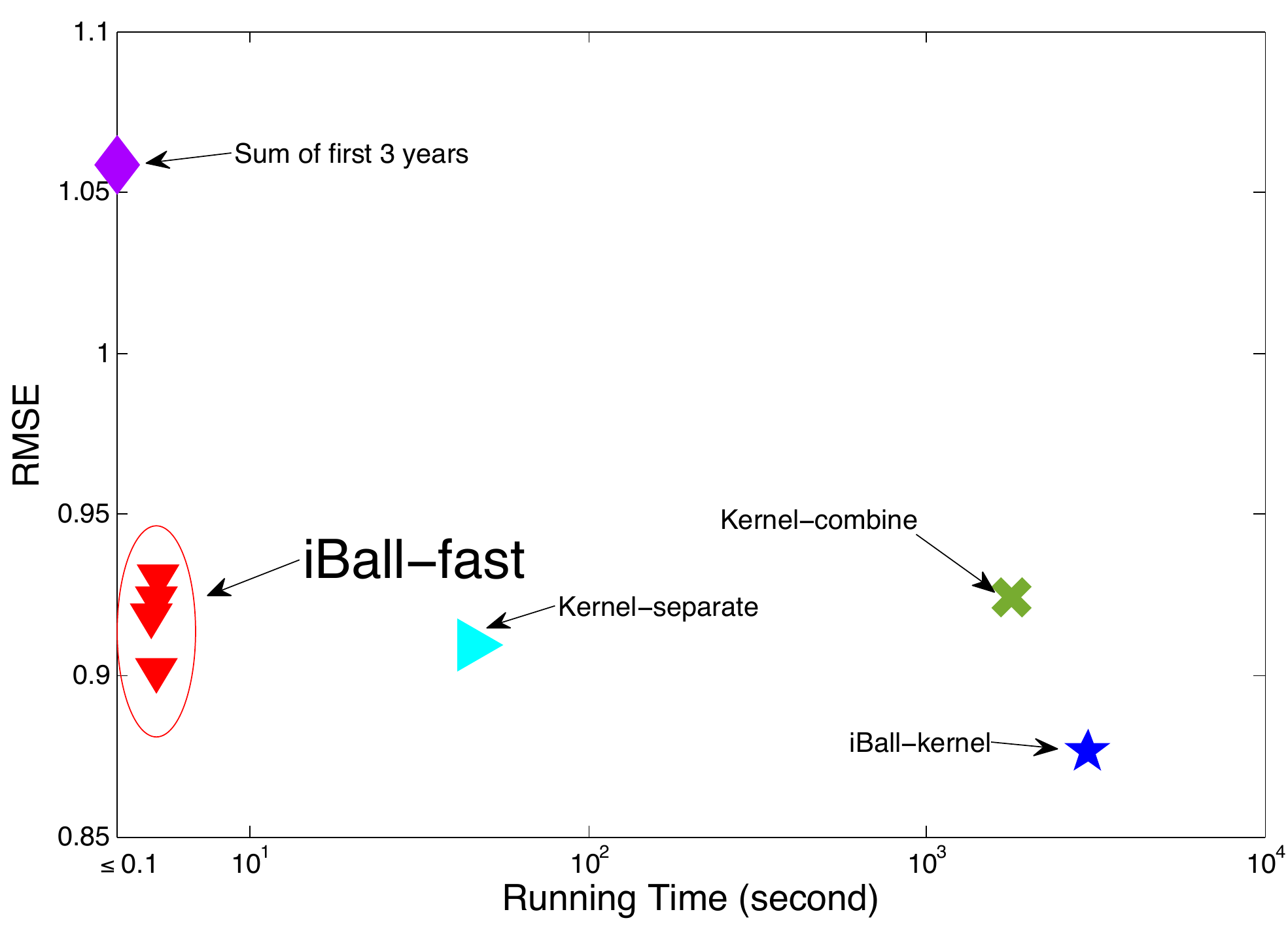}
\caption{Quality vs. speed with 88,905 training samples. }\label{fig:quality_plot}
\endminipage
\vspace{-5pt}
\end{figure}


\noindent {\it B - Quality vs. speed:} Finally, we evaluate how the proposed methods balance between the prediction quality and speed. In Figure~\ref{fig:quality_plot}, we show the RMSE vs. running time of different methods with 88,905 total training samples. For \ourAlgName-fast, we show its results using different rank $r$ for the low-rank approximation. Clearly, \ourAlgName-fast achieves the best trade-off between quality and speed as its results all lie in the bottom left corner.

\vspace{-5pt}
\section{Related Work}
\label{sec:rel}
In this section, we review the related work.

{\bf Impact/popularity prediction:}
As a pilot study, Yan et al.~\cite{DBLP:conf/cikm/YanTLSL11, DBLP:conf/jcdl/YanHTZL12} identify effective features to address citation count prediction problem. 
Davletov et al.~\cite{DBLP:conf/cikm/DavletovAC14} address the same problem by first clustering papers according to their temporal change in citation counts over time and assigning a polynomial to each cluster for regression.
In light of the difficulty posed by power law distribution of citations, Dong et al.~\cite{conf/wsdm/dong15} instead consider whether a paper can increase the primary author's {\it h}-index. Yu et al.~\cite{DBLP:conf/sdm/YuGZH12} address predicting citation relations in heterogeneous bibliographical networks.

A close line of work is to predict the popularity of other online contents, e.g., posts, videos, TV series.
Yao et al.~\cite{DBLP:conf/kdd/YaoTXL14} predict the long-term impact of questions/answers. Notice that in terms of methodology, the method in ~\cite{DBLP:conf/kdd/YaoTXL14} can be conceptually viewed as a special case of our \ourAlgName\ model when there are only two domains and the instance-level correspondence across different domains (e.g., question-answers association) is known.  
Li et al.~\cite{DBLP:conf/cikm/LiMWLX13} conduct an study on popularity forecast of videos shared in social networks. They consider both the intrinsic attractiveness of a video and the influence from the underlying diffusion structure. 
Chang et al.~\cite{DBLP:conf/cikm/ChangZGCXT14} are the first to comprehensively study for predicting the popularity of online serials with autoregressive models. As online serials have strong sequence dependence and release date dependence, they develop an autoregressive model to capture the dynamic behaviors of audiences. Though the focus of this paper is to propose a tailored method to predict the long-term citation counts, our method could be naturally applied to other related applications, e.g., popularity prediction.

{\bf Multi-task learning:} 
Our joint model is also related to multi-task learning as we jointly learn the models for each domain (task). Multi-task learning aims to improve the generalization performance of a learning task with the help of other related tasks. A key challenge in multi-task learning is how to exploit the relationship among different tasks to allow information shared across tasks. One way is by sharing of parameters. In neural networks, hidden units are shared across tasks~\cite{caruana1997multitask}.  It can also be induced by assuming that the parameters used by all tasks are close to each other by minimizing the Frobenius norms of their differences in methods based on convex optimization formulations~\cite{evgeniou2004regularized}. In Bayesian hierarchical models, parameter sharing can be imposed by assuming a common prior they share~\cite{yu2005learning}. A second way is assuming a common basis of the parameter space. A low-rank and sparse structure of the underlying predictive hypothesis has been applied to capture the tasks relatedness as well as outlier tasks~\cite{DBLP:conf/kdd/ChenLY10,Chen:2011:ILG:2020408.2020423,jalali2010dirty}.Our method is directly applicable when the correlation/similarity among different tasks is known and enjoys a closed-form solution. In terms of computation, we also provide an efficient way to track the joint predictive model in the dynamic setting.

{\bf Scholarly data mining:}
Scholarly data can be viewed as a heterogeneous information network of papers, authors, venues and terms~\cite{sun2009ranking}. Mining of such scholarly data is often from following perspectives: (1) similarity search to find a similar scholarly entities given a query entity or a set of query entities~\cite{DBLP:journals/pvldb/SunHYYW11,DBLP:conf/icdm/TongFP06,DBLP:conf/kdd/TongF06}; (2)literature recommendation to recommend related research papers on a topic~\cite{Liu:2014:MRP:2661829.2661965,chandrasekaran2008concept}; and (3)co-author collaboration prediction, to predict if two researchers will collaborate in the future~\cite{sun2011co,sun2012will}.

\vspace{-5pt}
\section{Conclusions}
\label{sec:con}
In this paper, we propose \ourAlgName ~-- a family of algorithms for the prediction of long-term impact of scientific work given its citation history in the first few years.
The proposed algorithms collectively address a number of key algorithmic challenges in impact prediction (i.e., feature design, non-linearity, domain heterogeneity and dynamics). It is flexible and general in the sense that it can be generalized to both regression and classification models; and in both linear and non-linear formulations; it is scalable and adaptive to new training data. 

\hide{
Our main contributions include:
\begin{itemize}
\item {\bf Algorithm:} A family of algorithms are proposed to address long-term impact prediction for scientific work.
\item {\bf Proofs and analysis:} Proofs of correctness of closed-form solutions and eigen update are given; update quality and complexity analysis are provided.
\item {\bf Empirical evaluations:} Extensive experimental evaluations validate the effectiveness and efficiency of the proposed algorithms.
\end{itemize}
}
\vspace{-5pt}
\section{Acknowledgments}
We would like to thank Dr. Jie Tang for providing the dataset. This material is supported by the National Science Foundation under Grant No. IIS1017415, by the Army Research Laboratory under Cooperative Agreement Number W911NF-09-2-0053, by National Institutes of Health under the grant number R01LM011986, Region II University Transportation Center under the project number 49997-33 25.  

The content of the information in this document does not necessarily reflect the position or the policy of the Government, and no official endorsement should be inferred.  The U.S. Government is authorized to reproduce and distribute reprints for Government purposes notwithstanding any copyright notation here on.
%
\balance
\small
\bibliographystyle{abbrv}
\bibliography{references}  
%
%

\end{document}